%% file: main.tex
\documentclass{article}
\usepackage[final]{neurips_2021}
\usepackage[shortlabels]{enumitem}

\usepackage[utf8]{inputenc} % allow utf-8 input
\usepackage[english]{babel}
\usepackage[T1]{fontenc} 

\usepackage{microtype}
\usepackage{graphicx}
\usepackage{subfigure}
\usepackage{booktabs}
\usepackage{enumitem}
\usepackage{nicefrac}
\usepackage{amssymb}

\usepackage{hyperref}

\usepackage{amsthm}
\usepackage{amsfonts}
\usepackage{amsmath}
\usepackage{mathrsfs}

\setcitestyle{round}
\usepackage{bbm}
\usepackage{mathtools}

\usepackage{url}
\usepackage{dsfont}
\usepackage{booktabs}
\usepackage{makecell}
\usepackage{tabularx}
\usepackage{xcolor, colortbl}
\usepackage[normalem]{ulem}
\usepackage{soul}
\usepackage{prettyref}
\usepackage{wrapfig}

\definecolor{Gray}{gray}{0.85}
\newcolumntype{g}{>{\columncolor{Gray}}c}
\hypersetup{
  colorlinks   = true,
  urlcolor     = blue,
  linkcolor    = blue,
  citecolor   = blue
}

\input{macros}
\input{temp_macros}
\usepackage[toc, page, header]{appendix}
\title{On The Existence of The Adversarial Bayes Classifier (Extended Version)}

\author{Pranjal Awasthi\\
Google Research\ignore{ \& Rutgers University}\\
New York, NY 10011, USA\\
\texttt{pranjalawasthi@google.com} \And Natalie S. Frank\\ Courant Institute\\ New York, NY 10012\\
\texttt{nf1066@nyu.edu}
\And Mehryar Mohri\\
Google Research \& Courant Institute\\
New York, NY 10011, USA\\
\texttt{mohri@google.com}}
\begin{document}
    
\maketitle

\input{Sections/1-Introduction}
\input{Sections/2-related_work}

\input{Sections/3-problem_setup}
\input{Sections/4-main_results}

\input{Sections/5-measurability_overview}
\input{Sections/6-alternative_perturbations}
\input{Sections/7-conclusion}

\newpage
\section*{Acknowledgements}

    This work was partly funded by NSF CCF-1535987 and NSF IIS-1618662. Natalie Frank was supported in part by the Research Training Group in Modeling and Simulation funded by the National Science Foundation via grant RTG/DMS – 1646339. The authors would like to thank Laurent Meunier for pointing out a missing step %or omission
    in the proof of Theorem~\ref{th:existence_basic} in the earlier version of the manuscript, Ryan Murray for finding an error in one of our proofs. The authors would also like to thank Professor James A. Morrow for a helpful citation \citep{nishiura2010} on measurability. 

\bibliographystyle{abbrvnat}
\bibliography{bibliography.bib,bib2.bib,bib3.bib,bib4.bib}
 
\appendix
\onecolumn

\renewcommand{\contentsname}{Contents of Appendix}
\tableofcontents
\addtocontents{toc}{\protect\setcounter{tocdepth}{3}} 
\clearpage

\input{Appendices/1-measurability}
\input{Appendices/2-set_limit_exists}

\input{Appendices/3-e-e_properties}

\input{Appendices/4-regularity_proof}

\input{Appendices/5-liminf_limsup_swap}
\input{Appendices/6-generalized_perturbations}

\end{document}

%% file: macros.tex
\newcommand{\cA}{\mathcal{A}}
\newcommand{\cB}{{\mathcal B}}

\newcommand{\cH}{\mathcal{H}}

\newcommand{\cL}{\mathcal{L}}

\newcommand{\sM}{{\mathscr M}}

\newcommand{\sU}{{\mathscr U}}

\newcommand{\bR}{{\mathbf R}}

\newcommand{\ba}{{\mathbf a}}

\newcommand{\bh}{{\mathbf h}}

\newcommand{\bs}{{\mathbf s}}

\newcommand{\bv}{{\mathbf v}}

\newcommand{\bx}{{\mathbf x}}
\newcommand{\by}{{\mathbf y}}
\newcommand{\bz}{{\mathbf z}}

\newcommand{\ignore}[1]{}
\newcommand{\zero}{\mathbf{0}}

\def\Rset{\mathbb{R}}
\def\Nset{\mathbb{N}}

\DeclareMathOperator*{\E}{\mathbb E}

\DeclareMathOperator{\sgn}{sign}
\DeclareMathOperator{\interior}{int}

\DeclareMathOperator{\tli}{\wt{\liminf}}
\DeclareMathOperator{\tls}{\wt{\limsup}}

\DeclarePairedDelimiter{\bracket}{[}{]}

\DeclarePairedDelimiter{\paren}{(}{)}

\newcommand{\mres}{\mathbin{\vrule height 1.6ex depth 0pt width
0.13ex\vrule height 0.13ex depth 0pt width 1.3ex}}

\newcommand{\PP}{{\mathbb P}}
\newcommand{\QQ}{{\mathbb Q}}

\newcommand{\dequal}{{\dot =}}

\newcommand{\one}{\mathbbm{1}}

\newcommand{\ov}{\overline}
\newcommand{\td}{\tilde}
\newcommand{\wt}{\widetilde}
\newcommand{\e}{\epsilon}

\newcommand{\set}[2][]{#1 \{ #2 #1 \} }

\newtheorem{theorem}{Theorem}
\newtheorem{definition}{Definition}
\newtheorem{lemma}{Lemma}

\newtheorem{example}{Example}

\makeatletter
\newtheorem*{rep@theorem}{\rep@title}
\newcommand{\newreptheorem}[2]{%
\newenvironment{rep#1}[1]{%
\def\rep@title{#2 \ref{##1}}%
\begin{rep@theorem}}%
{\end{rep@theorem}}}
\makeatother

\newreptheorem{lemma}{Lemma}
\newreptheorem{theorem}{Theorem}
\newreptheorem{corollary}{Corollary}
\newreptheorem{definition}{Definition}

%% file: temp_macros.tex
\usepackage{color}
\definecolor{jgGreen}{rgb}{0.0, 0.5, 0.0}
\definecolor{blue}{rgb}{0.0, 0.0, 1.0}

%% file: Sections/1-Introduction.tex
\begin{abstract}
Adversarial robustness is a critical property in a variety of modern
machine learning applications. While it has been the subject of
several recent theoretical studies, many important questions related
to adversarial robustness are still open.  In this work, we study a
fundamental question regarding Bayes optimality for adversarial
robustness. We provide general sufficient conditions under which the
existence of a Bayes optimal classifier can be guaranteed for
adversarial robustness. Our results can provide a useful tool for a
subsequent study of surrogate losses in adversarial robustness and
their consistency properties. This manuscript is the extended and corrected version of the paper \emph{On the Existence of the Adversarial Bayes Classifier} published in NeurIPS 2021. There were two errors in theorem statements in the original paper--- one in the definition of pseudo-certifiable robustness and the other in the measurability of $A^\e$ for arbitrary metric spaces. In this version we correct the errors. Furthermore, the results of the original paper did not apply to some non-strictly convex norms and here we extend our results to all possible norms.
\end{abstract}

\section{Introduction}
\label{sec:intro}

A key problem with using neural networks is their
susceptibility to small perturbations: imperceptible changes to the
input at test time may result in an incorrect classification by the
network \citep{szegedy2013intriguing}. A slightly perturbed picture of
a dog could be misclassified as a hand-blower. The same
phenomenon appears with other types of data such as biosequences,
text, or speech. This problem has motivated a series of research
publications studying the design of \emph{adversarially robust}
algorithms, both from an empirical and a theoretical perspective
\citep{szegedy2013intriguing, biggio2013evasion, madry2017towards,
  schmidt2018adversarially,
  athalye2018obfuscated,bubeck2018adversarial, montasser2019vc}.

In the context of classification problems, instead of the standard
zero-one loss, an \emph{adversarial zero-one loss} has been 
adopted which
penalizes a classifier not only if it misclassifies an input $x$ but
also if it does not maintain the correct $x$-label in a
$\e$-neighborhood around $x$ \citep{goodfellow2014explaining,
  madry2017towards, tsipras2018robustness, carlini2017towards}.
Since optimizing the adversarial zero-one loss is computationally
intractable, a common approach for adversarial learning is to use a
surrogate loss instead.  However, optimizing a surrogate loss over a
class of functions may not always lead to a minimizer of the true
underlying loss over that class. In the case of the standard zero-one
loss, there is a large body of literature identifying conditions under
which surrogate losses are \emph{consistent}, that is, minimizing them
over the family of all measurable functions leads to minimizers of the
true loss \citep{Zhang2003, BartlettJordanMcAuliffe2006,
  steinwart2005consistency, Lin2004}. More precisely, as argued by
\citet{long2013consistency}, it is in fact \emph{$\cH$-consistency}
that is needed, which is consistency restricted to the hypothesis set under consideration. A surrogate loss may be consistent for the family of all
measurable functions but not for the specific family
of functions $\cH$, and a surrogate loss can be $\cH$-consistent for a particular family $\cH$, without being consistent for
all measurable functions.

When are adversarial surrogate losses $\cH$-consistent? This problem
is already non-trivial for the standard zero-one loss: while
there are well-known results for the consistency of
losses for the zero-one loss such as \citep{BartlettJordanMcAuliffe2006, steinwart2005consistency}, these results do not hold for
$\cH$-consistency.
Existing theoretical results for $\cH$-consistency
assume that the Bayes risk is zero
\citep{long2013consistency,zhang2020bayes}. A similar situation seems
to hold for the more complex case of the adversarial loss. Recently,
\citet{AwasthiFrankMao2021} gave a detailed study of $\cH$-calibration and
$\cH$-consistency of surrogates to the adversarial loss and
also pointed out some technical issues with some $\cH$-consistency claims made in prior
work \citep{bao2020calibrated}. These authors presented a number of
negative results for adversarial $\cH$-consistency and positive results for some surrogate losses which assume realizability. For these positive
results, the zero Bayes adversarial loss seems necessary. In
fact, the authors show empirically that without the realizability
assumption, $\cH$-consistency does not hold for a variety of surrogate
losses, even when they are $\cH$-calibrated. 
% This
% requirement cannot be relaxed to an approximate zero Bayes adversarial
% loss as shown by that work's analysis and empirical results. Furthermore, the results of  \cite{bao2020calibrated} only apply to linear function classes because their definitions coincide with the adversarial loss only when the class $\cH$ is linear.

But when is the Bayes adversarial loss zero?
Clearly, the adversarial risk can only be zero if it admits a
minimizer, which we call the \emph{adversarial Bayes classifier}.
However, it is unclear under what conditions such a classifier exists. This is the primary theoretical question that we study in this work. 

We now describe the %some of the 
challenges involved in finding
minimizers of the adversarial zero-one loss. 
%To start, %in standard neural net training, the adversarial loss is discontinuous and the set of possible weights is unbounded and thus non-compact.
%Next, 
Most of the
existing work on the study of Bayes optimal classifiers focuses on
loss functions such as the zero-one loss that admit the
\emph{pointwise optimality} property \citep{steinwart2005consistency,
  steinwart2006function}. To illustrate this better, consider the case
of binary classification where on a given input $x$, $\eta(x)$ denotes
the conditional class probability, that is, $\eta(x) \coloneqq \PP(y =
1 \mid x)$. In this case, it is well-known that the Bayes optimal
classifier can be obtained by making optimal predictions per point in
the domain: at a point $x$ predict $1$ if $\eta(x) \geq \frac{1}{2}$,
$-1$ otherwise. Similar to the notion of a Bayes optimal classifier,
an{ adversarial Bayes optimal} classifier is the one that
minimizes the adversarial loss. However, an immediate obstacle is that the pointwise
optimality property does not hold for adversarial losses.

As an example, consider the case of binary classification and
perturbations measured in the $\ell_2$ norm. Then, for a given labeled
point $(x, y)$ and a perturbation radius $\epsilon$, the adversarial
zero-one loss of a classifier $f\colon \Rset^d\to \{-1,+1\}$ is defined as $\max_{x' \colon \|x'
  - x\|_2 \leq \epsilon} \one(f(x') \neq y)$. Thus, the loss at a
point $x$ cannot be measured simply by inspecting the prediction of
the classifier at $x$.  In other words, the construction of an
adversarial Bayes optimal classifier necessarily involves arguing
about the global patterns in the predictions of the classifier across
the entire input domain. As a result, most of the technical tools
developed for the study of Bayes optimal classifiers for traditional
loss functions are not applicable to the analysis of adversarial loss
functions, and new mathematical techniques are required.

The above discussion leads to our second motivation for studying the
question of existence of the adversarial Bayes classifier. Insights regarding the structure of the
adversarial Bayes optimal classifier could have
algorithmic implications. For example, in the case of the standard
zero-one loss, many popular learning algorithms seek to approximate
the conditional probability of a class at a point because the
conditional probability defines the Bayes optimal classifier in this
case. Analogously, one could hope to develop new algorithmic
techniques for adversarial learning with a better understanding of the
properties of adversarial Bayes classifiers. In fact, two recent
publications propose this approach \citep{yang2020,bhattacharjee2020}. Although their results do not rely on the existence of the adversarial Bayes classifier, they implicitly make this assumption to make their arguments clearer. Our work provides a rigorous basis for this premise. 

A second related concept is \emph{certified robustness}. A point $x$ is certifiably robust for a classifier $f$ and a perturbation radius $\e$ if \emph{every} perturbation of radius at most $\e$ leaves the class of $x$ unchanged. In this paper, we further study a property which we refer to as \emph{pseudo-certified robustness}, which is necessary for certified robustness. We show that there always exists an adversarial Bayes classifier which satisfies the pseudo-certified robustness condition for a fixed radius at every point. However, a non-trivial classifier cannot be certifiably robust for a fixed radius at every point -- specifically, a classifier is not certifiably robust at points within $\e$ of the decision boundary.

The concept of certified robustness has algorithmic implications. \cite{CohenRosenfeld2019} recently showed that after training a classifier, a process called \emph{randomized smoothing} makes the classifier certifiably robust at a point $x$ in the $\ell_2$ norm with a radius that depends on the point $x$.
As an adversarial Bayes classifier can be pseudo-certifiably robust but not certifiably robust with a fixed radius at every point, one could try to design algorithms which ensure pseudo-certifiable robustness during or after training. Recent works have explored constructing certificates of robustness as well \citep{Raghunathan2018,WengZhangChenSong2018,ZhangWengChenHsieh2018,WongKolter2018}. A better understanding of the adversarial Bayes classifier could help find additional learning algorithms. 
By
studying the existence of the adversarial Bayes classifier, we take a
first step towards this broader goal. 

We now describe the organization of the paper. Section~\ref{sec:related} summarizes related work and Section~\ref{sec:problem_setup} presents the mathematical formulation of our problem. Section~\ref{sec:main} discusses our main result and the proof. Next,
Section~\ref{sec:meas_discussion} addresses the measurability issues relating to this problem. 
Section~\ref{sec:alt_perturb} demonstrates how our techniques might apply to other models of perturbations. 
%The appendix includes the proofs of all intermediate results and two extensions of our main theorem. Appendices~\ref{app:measurability},~\ref{app:prokhorov},and~\ref{app:set_limit_proof} prove stand-alone statements, while Appendices~\ref{app:sequence_original}-\ref{app:generalize_result} rely upon the results in Appendix~\ref{app:properties}.
Subsequently, in Appendix~\ref{app:measurability}, we prove the measurability results stated in Section~\ref{sec:meas_discussion} and describe a similar result for metric spaces. Next, in Appendix~\ref{app:set_limit_proof}, we prove one of our key lemmas about convergence of sets. These appendices present stand-alone results which do depend on material elsewhere in the appendix.
In Appendix~\ref{app:properties}, we subsequently provide some background material for the results in Appendicies~\ref{app:sequence_original}-\ref{app:generalize_result}. Next, we prove the rest of our key lemmas in Appendicies~\ref{app:sequence_original} and~\ref{app:to_decreasing_sequence_proof}. Lastly, Appendix~\ref{app:generalize_result} states and proves two generalizations of our main result.

%% file: Sections/2-related_work.tex
\section{Related Work}
\label{sec:related}
Existing theoretical work on adversarial robustness focuses on
questions such as adversarial counterparts of VC-dimension and
Rademacher complexity \citep{cullina2018pac, khim2018adversarial,
  YinRamchandranBartlett2019, awasthi2020adversarial}, evidence of
computational barriers \citep{bubeck2018adversarial,
  bubeck2018adversarial2, nakkiran2019adversarial,
  degwekar2019computational} and statistical barriers towards ensuring
low adversarial test error \citep{tsipras2018robustness}.

\citet{cullina2018pac} formulate a notion of adversarial VC-dimension,
aimed at capturing uniform convergence of robust empirical risk
minimization. The authors show that, for linear models, adversarial
VC-dimension coincides with the VC-dimension. However, in general, the
two could be arbitrarily separate.  In a similar vein,
\citet{khim2018adversarial}, \citet{YinRamchandranBartlett2019} and 
  \citet{awasthi2020adversarial} study the Rademacher complexity of
adversarially robust losses for binary and multi-class
classification. \citet{schmidt2018adversarially} provide an instance
of a learning problem where one can provably demonstrate a gap between
the sample complexity of (standard) learning and adversarial learning.

\citet{tsipras2018robustness} points out a
problem where any learning algorithm that achieves low
(standard) test error must necessarily admit high
adversarial test error, that is close to $1$. This highlights a
fundamental tension between ensuring low test error and low
adversarial error. There are also studies of the conditions on the
data distribution that lead to the presence of adversarial examples
and the design of adversaries that can exploit them
\citep{diochnos2018adversarial,BartlettBubeckCherapanamjeri2021}. The recent work of
\citet{montasser2019vc} shows that any function class with finite
VC-dimension $d$ can be adversarially robustly learned (in a PAC-style
model) using $\exp(d)$ many samples. 

\citet{bubeck2018adversarial,
  bubeck2018adversarial2} provide evidence of computational barriers
in adversarial learning by constructing learning tasks that are easy
in the PAC model, but that become intractable when adversarial
robustness is required. Several recent publications have studied the question of
characterizing the Bayes adversarial risk \citep{pydi2019adversarial,
  bhagoji2019lower} for binary classification and relate it to the
optimal transportation cost between the two class conditional
distributions. While these studies aim to establish a lower bound on the
Bayes adversarial risk, we study a more fundamental question of when
the Bayes adversarial classifier exists. There have also been publications
studying robustness beyond $\ell_p$ norm perturbations
\citep{feige2015learning,feige2018robust, attias2018improved}.  

Finally, there are studies in the mathematical community of
various properties regarding the direct sum of a set and an $\e$-ball, which we use to model adversarial perturbations.  Similar, but
not identical mathematical constructions have also appeared in the PDE
literature.  \citet{cesaroni2017} and \citet{cesaroni2018} consider
perturbations to the measure-theoretic boundary of a set. However, the
measure-theoretic boundary and the topological boundary behave quite
differently.  \citet{chambolle2012} consider problems
involving integrals of indicator functions of perturbed sets $A^\e$
divided by the size of the perturbation. Additionally, 
\citet{bellettini2004} and \citet{Chambolle2015} assume some set properties that are satisfied by sets perturbed by $\ell_p$ balls, and then use
these to show
% properties to draw further conclusions, such as statements about 
 regularity and the curvature of the boundary. 
%  They later use these properties in analysis of partial differential equations.
Lastly, \citet{BertsekasShreve96} study the universal $\sigma$-algebra in detail, however they did not show that the sets we use in this paper are universally measurable. We prove a new measurability result in Section~\ref{sec:meas_discussion}.

%% file: Sections/3-problem_setup.tex
\section{Problem Setup}
\label{sec:problem_setup}

We study binary classification with class labels in $\{-1,+1\}$.
%We first describe the problem we consider in the special case of
%binary classification, with one class labeled with $-1$ and the other
%one $+1$.
We consider a probability distribution $\mathcal D$ over
$\Rset^d\times \{-1, +1\}$.
For convenience, $\eta$ will denote the conditional distribution, $\eta(\bx) = \mathcal D(Y = +1 | \bx)$ for any $\bx \in \Rset^d$, and $\PP$ will denote the marginal,
$\PP(A) = \mathcal D(A\times \{-1,+1\})$ for any measurable set $A \subseteq \Rset^d$.
% A common formulation of the problem then consists
% of learning 
Let $f \colon \Rset^d \to \Rset$ be a function whose sign defines a classifier. Then, for a perturbation set $B$, the \emph{adversarial loss} of $f$ is defined as
\begin{equation*}
\label{eq:adverarial risk}
R^\e(f)= \E_{(x,y) \sim \mathcal D} \bracket*{\sup_{\bh \in B} \one_{y\sgn(f(\bx+\bh))<0} } \quad \text{where} \quad \sgn(z)=
\begin{cases}
+1&\text{if }z>0\\
-1&\text{otherwise}
\end{cases}.
\end{equation*}
The adversarial loss has been extensively studied in recent years \citep{montasser2019vc, tsipras2018robustness, bubeck2018adversarial, khim2018adversarial, YinRamchandranBartlett2019}, motivated by the empirical phenomenon of adversarial examples \citep{szegedy2013intriguing}. In the rest of the paper, we will find it more convenient to work with an alternative set-based definition of classifiers~(and adversarial losses), which we describe below. The function $f$ induces two complementary sets $A = \set{\bx
  \colon f(\bx)> 0 }$ and $A^C = \set{\bx\colon f(\bx)\leq 0}$. 
Conversely, specifying the set $A$ is equivalent to specifying a
function $f$ since one could choose $f(\bx) = \one_A(\bx)$. In the
remainder of the paper, we will specify the set of points $A$ classified as
$+1$ rather than the function $f$.  The classification
risk of a set $A$ is then expressed as
\begin{equation}
\label{eq:bayes_objective}
R(A)
= \int (1 - \eta(\bx)) \one_A(\bx) + \eta(\bx) \one_{A^C}(\bx) \, d\PP.
\end{equation} 
    In the above formulation, it is easy to see that a Bayes optimal classifier is the set
$A = \{\bx\colon \eta(\bx)> \frac 12\}$. We now extend this viewpoint to adversarial losses.  We assume
that the adversary knows the classification set $A$ and that the
adversary seeks to perturb each point in $\Rset^d$ %so that it falls
outside of $A$, via an additive perturbation in a set $B$. 
In typical applications, $B$ is a ball in some norm, and in the rest of the paper we will assume that $B=\overline{B_\e(\zero)}$ is a closed  ball with radius $\e$ centered at the origin. Next, we define
$A^\e$ to be the set of points that
can fall inside $A$ after an additive perturbation of magnitude
at most $\epsilon$. Formally, $A^\e = \{\bx \in \Rset^d \colon \exists \bh \in \overline{B_\e(\zero)} \text{ for which }\bx+\bh\in A\}$.
Therefore, we can define the adversarial risk as 
\begin{equation}
\label{eq:adversarial_loss} R^\e(A) = \int \paren*{1 - \eta(\bx)} \one_{A^\e}(\bx) + \eta(\bx) \one_{(A^C)^\e}(\bx) \, d\PP.\end{equation}

\citet{pydi2019adversarial,bhagoji2019lower} also studied the
adversarial Bayes classifiers using the $\empty^\e$
operation. We will now re-write $A^\e$ in a form more amenable to analysis:
\begin{align*}
A^\e & = \{ \bx\in \Rset^d\colon \exists \bh \in \overline{B_\e(\zero)}| \bx+\bh\in A\}=\{ \bx\in \Rset^d\colon \exists \bh \in \overline{B_\e(\zero)}\text{ and }\ba\in A| \bx+\bh=\ba \}\\
&= \set[\Big]{\bx \colon \exists \bh\in \overline{B_\e(\zero)}\text{ and } \ba \in A \mid
 \ba - \bh = \bx}=\{\ba-\bh\colon \ba\in A, \bh\in \overline{B_\e(\zero)} \}=A\oplus  \overline{B_\e(\zero)},
\end{align*}
where the last equality follows from the symmetry of the ball $\overline{B_\e(\zero)}$. From these relations, we can recover a more typical expression of the adversarial loss. Note that 
$\one_{A^\e}(\bx)=\one_{A\oplus \overline{B_\e(\zero)}}(\bx)=\sup_{\bh\in \overline{B_\e(\zero)}} \one_{A}(\bx+\bh)$,
which implies
\begin{align}
R^\e(A)
&= \int (1 - \eta(\bx)) \sup_{\bh\in \overline{B_\e(\zero)}}\one_A(\bx+\bh) + \eta(\bx) \sup_{\bh\in \overline{B_\e(\zero)}}\one_{A^C}(\bx+\bh) \, d\PP. \label{eq:bayes_adv_objective_sup}
\end{align} 

The papers
\citep{szegedy2013intriguing,biggio2013evasion,madry2017towards} (and
many others) use the multi-class version of this loss to define adversarial
risk. More specifically, they evaluate the risk on the set
$A=\{f(\bx)\geq 0\}$, where $f$ is a function in their model class.  

We define the {\em adversarial Bayes risk} $R^\e_*$ as the infimum
of \eqref{eq:adversarial_loss} over all measurable sets, and we say
that the set $A$ is an adversarial Bayes classifier if
$R^\e(A)=R^\e_*$. Note that the integral above is defined only if the
sets $A^\e, (A^C)^\e$ are measurable. This consideration is nontrivial
as there do exist measurable sets whose direct sum is not
measurable, see \citep{ErdosStone1970, CiesielskiFejzicFreiling} for
examples. 
\setlength{\intextsep}{-1pt}
\setlength{\columnsep}{6pt}
\begin{wrapfigure}{r}{0.275\textwidth}
\centering
\includegraphics[width=.275\textwidth]{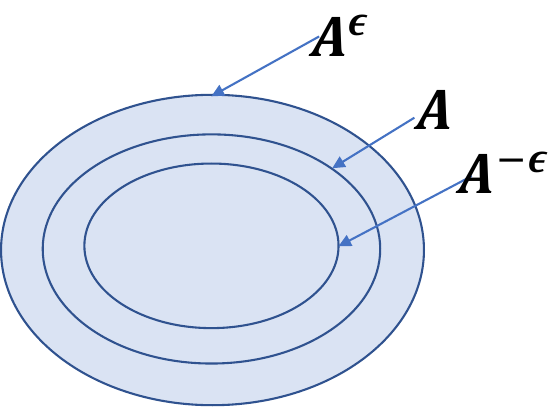}
\vskip -.1in
\caption{Sets $A^\e$ and $A^{-\e}$ with $B=\overline{B_\e^2(\zero)}$, the closed $\ell_2$ ball.}
%\vskip -.15in
\label{fig:adv_sets}
\end{wrapfigure}
To address this issue, in Section~\ref{sec:meas_discussion}, we discuss a
$\sigma$-algebra called the \emph{universal $\sigma$-algebra} which is
denoted $\sU(\Rset^d)$.  Specifically, we show that if $A\in
\sU(\Rset^d)$, then $A^\e\in \sU(\Rset^d)$ as well. Thus, working in
the universal $\sigma$-algebra $\sU(\Rset^d)$ allows us to define the
integral in \eqref{eq:adversarial_loss} and then optimize $R^\e$ over
sets in $\sU(\Rset^d)$. In particular, throughout this paper, we adopt
the convention that $\PP$ is the completion of a Borel measure restricted to $\sU(\Rset^d)$. (We elaborate on this
construction in Section~\ref{sec:meas_discussion}.) We
call a set \emph{universally measurable} if it is in the universal
$\sigma$-algebra $\sU(\Rset^d)$.

We now introduce another important notation: we define
$A^{-\e}\colon = ((A^C)^\e)^C$. The set $A^{-\e}$
contains the points that cannot be perturbed to fall outside of
$A$ (see Lemma~\ref{lemma:more_about_-eps} in Appendix~\ref{app:properties} for a formal proof).
Figure~\ref{fig:adv_sets} depicts the sets $A,A^\e$ and $A^{-\e}$.  

%% file: Sections/4-main_results.tex
\section{Main Results}
\label{sec:main}

In this section, we prove our main result establishing the existence
of the optimal adversarial classifier. We first discuss challenges in
establishing this theorem.
% We start with comparing to the existence of the Bayes classifier to
% explain why this result is difficult.
In the case of the standard 0-1 loss, the risk is defined in
\eqref{eq:bayes_objective} with the sets $A$ and $A^C$
disjoint. As a result,
% standard scenario, one could define risk as in
% \eqref{eq:bayes_objective}.  As $A$ and $A^C$ are disjoint,
the integrand equals either $\eta(\bx)$ or $(1-\eta(\bx))$ at each
point. Thus  the set for which $1-\eta(\bx)<
\eta(\bx)$ minimizes $R$. In other words, the Bayes
classifier minimizes the objective $\min(\eta(\bx),1-\eta(\bx))$ at
each point.

On the other hand, the same reasoning does not apply to the
adversarial risk. The adversarial risk at a single point $\bx$ depends
on all the points in $\overline{B_{\e}(\bx)}$. Hence, one cannot hope to
find the adversarial Bayes classifier by studying the risk in a
pointwise manner. 

Next, we introduce the concepts of certifiable robustness and pseudo-certifiable robustness. 
    \begin{definition}\footnote{Pseudo-certifiable robustness was defined differently in the original version of this paper. We thank Ryan Murray for pointing out an error in Theorem~\ref{th:existence_basic} stemming from the earlier version of this definition.}
        Fix a perturbation radius $\e$. We say that a classifier $A$ is \emph{certifiably robust at a point $\bx$ with radius $\e$} if either $\bx\in A$ and $\ov{B_\e(\bx)}\subset A$, or $\bx\in A^C$ and $\ov{B_\e(\bx)}\subset A^C$. We say that a classifier $A$ is \emph{pseudo-certifiably robust at a point $\bx\in A$ with radius $\e$} if there exists a ball $\ov{B_\e(\by)}$ with $\bx\in \overline{B_\e(\by)}$ and $\ov{B_\e(\by)}\subset A$. We say a classifier $A$ is \emph{pseudo-certifiably robust} if it is pseudo-certifiably robust with radius $\e$ at every point.
    \end{definition}

In other words, a classifier is certifiably robust at a point $\bx\in A$ with radius $\e$ if the entire $\e$-ball around $\bx$ is classified the same as $A$, and a classifier is pseudo-certifiably robust at a point $\bx\in A$ with radius $\e$ if \emph{some} closed $\e$-ball radius containing $\bx$ is included in $A$. Pseudo-certifiable robustness is a necessary condition for certifiable robustness. 

We now discuss potential algorithmic applications of pseudo-certifiable robustness. To begin, we start by defining the set of points at which a classifier is not pseudo-certifiably robust. If we define
\begin{equation}
F(A)=\{\bx\in A:\text{ every closed $\e$-ball containing $\bx$ also intersects }A^C\}\label{eq:define_F(A)}.
\end{equation}

\setlength{\intextsep}{0pt}
\setlength{\columnsep}{5pt}
\begin{wrapfigure}{t}{.5\textwidth}
\centering
\vskip -.05in
\includegraphics[scale=0.435]{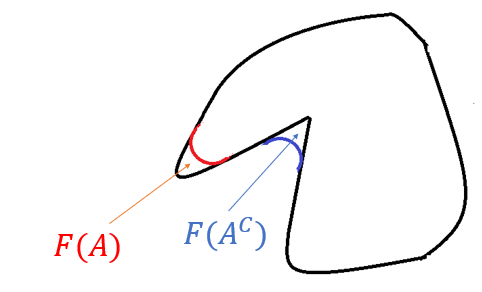}
\vskip -.15in
\caption{The figure illustrates a set $A$ with the sets $F(A)$ and $F(A^C)$ roughly indicated. For a point $\ba\in F(A)$, every closed $\e$-ball containing $\ba$ also intersects $A^C$ while for $\ba \in F(A^C)$ every closed $\e$-ball containing $\ba$ also intersects $A$.}
\label{fig:F(A)_F(A^C)}
\vskip -.15in
\end{wrapfigure}
Then, the set of points where a classifier is not pseudo-certifiably robust is $F(A)$. 
In Appendix~\ref{app:sequence_original}, we show that ``subtracting" from a classifier the points at which it is not pseudo-certifiably robust can only reduce the risk. Similarly, ``adding" to a classifier $A$ the points at which $A^C$ is not pseudo-certifiably robust can only reduce the risk as well. Formally, we show that $R^\e(A-F(A))\leq R^\e(A)$ and $R^\e(A\cup F(A^C)\leq R^\e(A)$ (Lemma~\ref{lemma:decreaseR_e(A)}). As illustrated in Figure~\ref{fig:F(A)_F(A^C)}, $F(A),F(A^C)$ are adjacent to the boundary $\partial A$. Furthermore, $F(A)$ is not very ``large"--- in fact, $F(A)^{-\e}=\emptyset$. These observations suggest that, typically, if $A$ is not pseudo-certifiably robust, then there is another classifier with lower risk that can be found by making local changes to $A$.

We now state our main existence result. We define the measures $\PP_0$, $\PP_1$ as in \citep{pydi2019adversarial} as
\[\PP_1(A)=\int_A \eta d\PP,\quad \PP_0(A)=\int_A(1-\eta)d\PP\]

\begin{theorem}
\label{th:existence_basic} \footnote{The original version of this paper did not assume that either $\PP_0$ or $\PP_1$ was absolutely continuous with respect to Lebesgue measure.}
Let $\PP$ be the completion of a Borel measure on $\mathcal
B(\Rset^d)$ restricted to $\sU(\Rset^d)$ and assume that either $\PP_0,\PP_1$ is absolutely continuous with respect to Lebesgue measure. Define $A^\e=A\oplus
\overline{B_\e(\zero)}$, where $B_\e(\zero)$ is a norm ball. Then, there exists a minimizer of
\eqref{eq:adversarial_loss} when minimizing over $\sU(\Rset^d)$. Furthermore, there exists a minimizer that is pseudo-certifiably robust and a minimizer whose complement is pseudo-certifiably robust.
\end{theorem}

The original version of this paper published in NeurIps \citep{AwasthiFrankMohri2021} proves this result for a restricted class of norms. For perturbations in an arbitrary norm, the theorem provides a positive guarantee:
for any distribution, the adversarial Bayes classifier exists. In fact the proof of Theorem~\ref{th:existence_basic} shows that an even stronger result holds: under the hypotheses of Theorem~\ref{th:existence_basic}, every minimizing sequence $A_{n}$ has a subsequece $A_{n_j}$ for which $\limsup_j A_{n_j}$ is the adversarial Bayes classifier. This conclusion is analogous to saying that every minimizing sequence must have a convergent subsequence.
 
 To understand the significance of this statement, we compare to minimizing a function over $\Rset$. Consider the three functions $f(x)=(x^2-1)^2$, $g(x)=\sin(x)^2$, and $h(x)=1/x^2$. The infimum of all three functions is 0. We can find minimizing sequences for $f,g$, and $h$ which don't converge. For instance, the sequence given by
 \[x_k=
    \begin{cases}
        +1 &k \text{ even}\\
        -1&k \text{ odd}
    \end{cases}\]
    is a minimizing subsequence of $f$ because $f(x_k)=0$ for all $k$, but $x_k$ is not a convergent subsequence. Intuitively, this phenomenon occurs because $x_i$ is actually comprised of two subsequences each of which converges to a minimizer of $f$. In this case, every minimizing sequence of $f$ has a convergent subsequence.
    On the other hand, minimizing sequences of $g$ have very different behavior. For instance, consider the sequence given by $y_k=k\pi.$
    Then $y_k$ is a minimizing sequence of $g$ because $g(y_k)=0$ for all $k$. However, the sequence $y_k$ diverges to infinity, so $\{y_k\}$ does not have any convergent subsequence.
        Lastly, the sequence $y_k$ also minimizes $h(x)$. Notably, $h$ does not have a minimizer and thus all minimizing sequences diverge.

We also expect
an analogous existence result for perturbations by open balls.

Next, we briefly discuss two ways in which our results relate to the consistency of adversarial losses. First, \citet{AwasthiFrankMao2021} show that when $H$ is the class of linear functions, if the surrogate risk $R^\e_\Psi$ of the adversarial surrogate loss $\Psi$ is zero for a given distribution, then $\Psi$ is $H$-consistent for that distribution. The existence of the adversarial Bayes classifier is required for this condition to hold. Next, a surrogate loss $\Psi$ is consistent if a minimizing sequence of functions $f_i$
also minimizes 0-1 adversarial loss. However, it may be easier to study minimizing sequences of the $\Psi$ loss when we have information about the adversarial Bayes classifier. The proof of Theorem~\ref{th:existence_basic} shows that under reasonable assumptions on $\eta$ and $\PP$, every sequence $A_n$ has a subsequence $A_{n_j}$ for which $\limsup A_{n_j}$ is an adversarial Bayes classifier. Thus, we can find conditions under which $\{\bx\colon f_i(\bx)\geq 0\}$
approaches a set $A$. In other words: If $\Psi$ is consistent and
$f_i$ is a sequence that minimizes the adversarial $\phi$ loss, then
$f_i\geq 0$ must have a subsequence that approaches an adversarial Bayes classifier. 

%The rest of the section discusses the proof of Theorem~\ref{th:existence_basic}.

\subsection{Proof strategy}
\label{sec:existence_basic_proof_strategy}

We first outline the main ideas behind the proof of
Theorem~\ref{th:existence_basic}, which is presented in the next
subsection. The proof applies the direct method of the calculus of
variations.  Specifically, we apply the following procedure:
\begin{description}[itemsep=-0mm]
\item[1)] Choose a sequence of sets $\set{A_n}\subset \sU(\Rset^d)$ 
along which $R^\e(A_n)$ approaches its infimum;

\item[2)] Extract a subsequence $\set{A_{n_j}}$ of 
$\set{A_n}$  that is convergent in some topology;

\item[3)] Show that $R^\e$ is sequentially lower semi-continuous: for a convergent subsequence $\{A_n\}$,
\[
\liminf_{n\to \infty} R^\e( A_n)\geq R^\e(\lim_{n\to \infty} A_n).
\]
\end{description}

In typical applications of the direct method, step 2) is almost
immediate as it is achieved by working in the appropriate Sobolev
space. However, showing step 3) is usually quite difficult. See \citet{Dacorogna2008} for more on the direct method in PDEs.
In contrast, in our scenario, the situation is the opposite: finding
the right topology for step 2) is quite difficult but the lower
semi-continuity is a direct implication of Fatou's lemma.

As described above, one of the main considerations in the proof of
Theorem~\ref{th:existence_basic} is the convergence of set sequences.
%We now discuss this aspect of the proof at a high level.
In order to find a minimizer, we need the
indicator functions $\one_{(A_n)^\e}, \one_{(A_n^C)^\e}$ to converge.
With that in mind, we adopt the following standard set-theoretic
definitions for a sequence of sets $\{A_n\}$:
\begin{equation}\label{eq:liminf_limsup)_def}
 \limsup A_n \!=\! \bigcap_{N \geq 1} \bigcup_{n\geq N} A_n 
\text{ and }  \liminf A_n \!=\! \bigcup_{N \geq 1} \bigcap_{n\geq N} A_n.
\end{equation}
 As with $\limsup$ and $\liminf$
for a sequences of numbers, $\liminf A_n \subset \limsup A_n$ or in other words $\one_{\liminf A_n}\leq \one_{\limsup A_n}$. With
the above definitions, the following holds:
\[
\liminf_{n\to \infty} \one_{A_n} 
= \one_{\liminf A_n}
\text{ and }\limsup_{n\to \infty} \one_{A_n}
= \one_{\limsup A_n}.
\] 
Specifically, these relations imply that the limit $ \lim_{n\to \infty}
\one_{A_n}$ exists $\PP$-a.e. if and only if the $\limsup$ and the $\liminf$
of the sequence $\{A_n\}$ match up to sets of measure zero under
$\PP$. We denote equality up to sets of Lebesgue measure zero by $\doteq$. In
order to find a sequence for which $\liminf A_n^\e \doteq \limsup A_n^\e$, we apply a theorem from variational analysis in \citep{RockafellarWets1998}. Specifically, we show

\begin{lemma}\label{lemma:set_limit_2}
    Let $\QQ$ be a finite positive measure and assume that $\QQ$ is absolutely continuous with respect to Lebesgue measure. For any sequence of sets $A_n$, there is a sub-sequence $A_{n_j}$ for which 
    \[\limsup A_{n_j}^\e\dequal \liminf A_{n_j}^\e\]
\end{lemma}

The lemma above is proved in Appendix~\ref{app:set_limit_proof}.
The next challenge is that $\liminf A_n^\e/ \limsup A_n^\e$ 
do not necessarily equal $A^\e$ for some set $A$. However, moving the $\empty^\e$ operation inside the $\liminf/\limsup$ decreases the risk.

\begin{lemma}\label{lemma:limsup_liminf_e_commute}
    Let $A_n$ be any sequence of sets. Then
    \[\limsup A_n^\e \supset \left( \limsup A_n\right)^\e \quad \text{and}\quad \liminf A_n^\e \supset \left( \liminf A_n\right)^\e\]
\end{lemma}
The lemma is proved in Appendix~\ref{app:sequence_original}. 

Finally, it remains to show the claim about pseudo-certifiable robustness. We prove that for any set $A$ there are sets $B$,$E$ for which $B$ and $E^C$ are pseudo-certifiably robust and have lower robust risk than $A$.

\begin{lemma}\label{lemma:to_pcr_set}
    Let $A$ be any set. Then there exist sets $B,E$ for which $B$ and $E^C$ are pseudo-certifiably robust and $R^\e(B)\leq R^\e(A)$, $R^\e(E)\leq R^\e(A)$.
\end{lemma}
To prove this result, we show that applying $\empty^{-\e}$, $\empty^\e$ in succession to a set removes $F(A)$ as defined in \eqref{eq:define_F(A)}. Analogously, applying $\empty^{-\e}, \empty^{\e}$ in succession to a set adds $F(A^C)$. We prove the above Lemma in Appendix \ref{app:to_decreasing_sequence_proof}. 
 
 \subsection{Formal Proof of Theorem~\ref{th:existence_basic}}\label{sec:main_proof}
 We now formally prove Theorem~\ref{th:existence_basic} using these three lemmas.

	\begin{proof}[Proof of Theorem~\ref{th:existence_basic}]
		WLOG assume that $\PP_1$ is absolutely continuous with respect to Lebesgue measure. 

        Let $A_n$ be a minimizing sequence of $R^\e$. By Lemma~\ref{lemma:set_limit_2}, there is a subsequence $A_{n_j}$ for which $\limsup_j A_{n_j}^\e\dequal \liminf_j A_{n_j}^\e$ and thus
        \begin{equation}\label{eq:liminf_limsup_equal}
            \int \eta \one_{\limsup_j A_{n_j}^\e} d\PP=\int \eta\one_{\liminf_j A_{n_j}^\e}d\PP
        \end{equation}
        
        Fatou's lemma then implies that 
        \begin{align*}
            \inf_A R^\e(A)&=\liminf_{j\to \infty} R^\e(A_{n_j})\geq \int\liminf_{j\to \infty} \left(\eta \one_{A_{n_j}^\e}+(1-\eta)\one_{(A_{n_j}^C)^\e}\right)d\PP\\
            &\geq \int\liminf_{j\to \infty} \eta \one_{A_{n_j}^\e}+\liminf_{j\to \infty} (1-\eta)\one_{(A_{n_j}^C)^\e}d\PP\\
            &=\int \eta \one_{\limsup_j A_{n_j}^\e}+(1-\eta)\one_{\liminf_j(A_{n_j}^C)^\e}d\PP &(\text{Equation~\ref{eq:liminf_limsup_equal})}\\
            &\geq \int \eta \one_{(\limsup_j A_{n_j})^\e}+(1-\eta)\one_{(\liminf_jA_{n_j}^C)^\e}d\PP &\text{(Lemma~\ref{lemma:limsup_liminf_e_commute})}\\
            &=\int \eta \one_{(\limsup_j A_{n_j})^\e}+(1-\eta)\one_{((\limsup_jA_{n_j})^C)^\e}d\PP
        \end{align*}
        Therefore, $A=\limsup_j A_{n_j}$ is a minimizer of $R^\e$. Lemma~\ref{lemma:to_pcr_set} then implies that there are sets $B,E$ for which $B,E^C$ are pseudo-certifiably robust and $R^\e(B)\leq R^\e(A)$ and $R^\e(E)\leq R^\e(A)$. Therefore, $B,E$ are minimizers as well.
	\end{proof}

 \subsection{Proof Outline for Lemmas~\ref{lemma:set_limit_2},~\ref{lemma:limsup_liminf_e_commute}, and~\ref{lemma:to_pcr_set}}
\label{sec:sequence_original_proof_strategy}

In this section, we explain the intuition for the proofs of Lemmas
~\ref{lemma:set_limit_2},~\ref{lemma:limsup_liminf_e_commute}, and~\ref{lemma:to_pcr_set}.
 Lemmas~\ref{lemma:set_limit_2} and~\ref{lemma:limsup_liminf_e_commute} follow directly from properties of the $\empty^\e$ operation. Specifically, in Appendix~\ref{app:properties_basic} we show that
 \begin{equation}
 \label{eq:union_prop_2}
\left(\bigcup_{i=1}^\infty A_i\right)^\e=\bigcup_{i=1}^\infty A_i^\e \quad \text{and} \quad
\left(\bigcap_{i=1}^\infty A_i\right)^\e\subset\bigcap_{i=1}^\infty A_i^\e. 
\end{equation}
As the $\liminf$ and $\limsup$ operations of \eqref{eq:liminf_limsup)_def} are defined by unions and intersections, this result immediately implies Lemma~\ref{lemma:limsup_liminf_e_commute}. Next, one can use the relations of \eqref{eq:union_prop_2} to argue that if $B=(A^{-\e})^{\e}$, then $(B^C)^{\e}=(A^C)^{\e}$ and $B^\e\subset A^\e$ so therefore $R^\e(B)\leq R^\e(A)$. One can make an analogous statement with $E=(A^\e)^{-\e}$, see Appendix~\ref{app:properties_succession} for the formal statement and proof.
 
The proof of Lemma~\ref{lemma:set_limit_2} combines the analysis of the $\empty^\e$ operation with measure theoretic considerations. \citet{RockafellarWets1998} prove a set convergence result for a different notion of the $\liminf$ and $\limsup$ of a sequence of sets $S_n$. This notion of set convergence includes points that are arbitrarily close to $S_n$ for infinitely many $n$. The standard $\liminf/\limsup$ operations have a similar interpretation in terms of sequences. Recall that
\begin{equation}\label{eq:liminf_sequence_def}
    \liminf S_n = \{\bx \colon \text{ there exists an $N$ for which $\bx \in S_n$ for all $n>N$}\}    
\end{equation}

\begin{equation}\label{eq:limsup_sequence_def}
    \limsup S_n=\{\bx\colon \text{ there exists a sequence $n_j$ for which } \bx\in S_{n_j} \text{ for all }j\}    
\end{equation}

On the other hand, the \citet{RockafellarWets1998} defines $\tli,\tls$ in terms of convergent sequences $\{\bx_n\}$ with $\bx_n \in S_n$:

	\[\tli S_n=\{\bx\colon \text{there exists a sequence with }\bx_n\in S_n \text{ and }\lim_{n\to \infty} \bx_n=\bx\}\]

    \begin{equation}\label{eq:tls_def}
        \tls S_n=\{\bx\colon \text{there exists a subsequence }\{n_j\} \text{ with }\bx_{n_j}\in S_{n_j} \text{ and }\lim_{i\to \infty} \bx_{n_j}=\bx\}    
    \end{equation}

In other words, a point $\bx$ is in $\liminf S_n$ if $\bx\in S_n$ for sufficiently large $n$ while  $\bx\in \tli S_n$ if the distance between $\bx$ and $S_{n_j}$ approaches zero. Similarly, a point is in $\limsup S_n$ if $\bx \in S_{n_j}$ for some subsequence $n_j$ while $\bx \in \tls S_{n_j}$ if there is a subsequence $n_j$ for which the distance between $\bx$ and $S_n$ approaches zero. This characterization immediately implies $\liminf S_n\subset \tli S_n$ and $\limsup S_n\subset \tls S_n$. For the notions of set limit $\tli,\tls$ every subsequence has a convergent subsequence:

	\begin{theorem}[\citep{RockafellarWets1998}]\label{th:set_limit_Rockafellar}
		Let $S_n$ be any sequence of sets in $\Rset^d$. Then there is a subsequence $S_{n_j}$ of $S_n$ for which $\tli S_{n_j}=\tls S_{n_j}$. 
		
	\end{theorem}
    This statement is a consequence of Theorem~4.18 of \citep{RockafellarWets1998}.
	In other words, one can always choose a subsequence $S_{n_k}$ of $S_{n}$ for which the $\tli$ and the $\tls$ match. This result is \emph{false} for the standard definitions of $\liminf$, $\limsup$.

    However, pseudo-certifiably robust sets are fairly well-behaved, so one would hope such sets would also interact well with the standard definition of $\liminf$ and $\limsup$. Furthermore, a standard argument from geometric measure theory implies that pseudo-certifiably robust sets have a measure zero boundary. 
        \begin{lemma}\label{lemma:regularity_ball}
        Let $\mu$ be Lebesgue measure and let $S\subset \Rset^d$. 
    If for each $\bs\in \partial S$ there exists a ball $B_\e(\ba)$ with $B_\e(\ba)\subset S$ and $\bs\in \partial B_\e(\ba)$, then $\mu(\partial S)=0$.
    \end{lemma}
    See Appendix~\ref{app:set_limit_lemmas} for a proof. 
    
    One can show that for a subsequence $A_{n_j}$ with $\tli A_{n_j}=\tls A_{n_j}$, the set $\liminf A_{n_j}^\e$ satisfies a property similar to pseudo-certifiable robustness: for all $\bx \in \liminf A_{n_j}^\e$, there is a ball $B_\e(\ba)$ for which $\bx \in \ov{B_\e(\ba)}$ and $B_\e(\ba)\subset \liminf A_{n_j}^\e$ (See Lemma~\ref{lemma:liminf_A^e_well_behaved} in Appendix~\ref{app:set_limit_proof}). In other words, the condition of Lemma~\ref{lemma:regularity_ball} is satisfied at every point in $\liminf A_{n_j}^\e$. By taking limits, one can then argue that this property also holds for all $\bx\in \partial \liminf A_{n_j}^\e$. 
    Lemma~\ref{lemma:regularity_ball} then implies Lemma~\ref{lemma:set_limit_2}.

%% file: Sections/5-measurability_overview.tex
\section{Addressing Measurability}
\label{sec:meas_discussion}

As mentioned earlier, defining the adversarial loss requires integrating 
over $A^\epsilon$. However, one must ensure that $A^\epsilon$ is measurable. 
% Furthermore, later in the paper, in order to find a minimizer of $R^\e$ with nice properties, we apply the $\empty^\e, \empty^{-\e}$ operations multiple times in succession. In particular, the proof of Lemma~\ref{lemma:sequence_original} investigates sets of the form $((A^{-\e})^{2\e})^{-\e}$. To handle these situations, we would like to work in a $\sigma$-algebra $\Sigma$ for which if $A\in \Sigma$, $A^\e \in \Sigma$ as well. Below we show that a $\sigma$-algebra called the \emph{universal $\sigma$-algebra} satisfies this property.
Furthermore, in the proof of Lemma~\ref{lemma:to_pcr_set}, we apply the $\empty^\e, \empty^{-\e}$ operations multiple times in succession. In particular, we consider sets of the form $(A^{-\e})^{\e}$. Hence we would like to work in a $\sigma$-algebra $\Sigma$ for which if $A\in \Sigma$, $A^\e \in \Sigma$ as well. Below, we explain that a $\sigma$-algebra called the \emph{universal $\sigma$-algebra} satisfies this property. 

% We define this concept here following the treatment of \citep{nishiura2010}. 

Let $\cB(\Rset^d)$ be the Borel $\sigma$-algebra on $\Rset^d$ and let $\nu$ be a measure on this $\sigma$-algebra. We will denote the completion of the measure space $(\nu, \Rset^d, \cB(\Rset^d))$ by $(\overline \nu, \Rset^d,\cL_\nu(\Rset^d)) $, where $\cL_\nu(\Rset^d)$ is the completion of $\cB(\Rset^d)$ under $\nu$. Let $\sM(\Rset^d)$ be the set of all finite Borel measures on $\Rset^d$. Then we define the \emph{universal $\sigma$-algebra} as $\sU(\Rset^d)=\bigcap_{\nu\in \sM(\Rset^d)} \cL_\nu(\Rset^d)$. 
In other words, $\sU(\Rset^d)$ is the sets which are measurable under \emph{every} complete finite Borel measure. For the universal $\sigma$-algebra, we have the following theorem proved in Appendix~\ref{app:meas_vector}:
\begin{theorem}\label{th:A^e_univ_meas} 
If $A\in \sU(\Rset^d)$, then $A^\e\in \sU(\Rset^d)$ as well.
\end{theorem}
 Specifically, Theorem \ref{th:A^e_univ_meas} allows us to define the adversarial risk in Equation \eqref{eq:adversarial_loss}. Appendix~\ref{app:meas_metric} proves a similar measurability theorem for metric spaces. Recall that for a probability measure $\QQ$, by definition $\sU(\Rset^d)\subset \cL_\QQ(\Rset^d)$. Therefore, if $A\in \sU(\Rset^d)$, then $A^\e$ is measurable with respect to $(\overline \QQ, \Rset^d, \cL_\QQ(\Rset^d))$. However, as this only holds for $A\in \sU(\Rset^d)$ and not all of $\cL_\QQ(\Rset^d)$, throughout this paper, \emph{we implicitly assume that our measure space is $(\overline \QQ, \Rset^d, \sU(\Rset^d))$}. In other words, we assume that the probability measure $\PP$ is a complete measure $\overline \QQ$ restricted to the $\sigma$-algebra $\sU(\Rset^d)$. As $\sU(\Rset^d)$ is closed under the $\empty^\e,\empty^{-\e}$ operations, this convention allows us to mostly ignore measurability considerations.
 
  Results similar to Theorem~\ref{th:A^e_univ_meas} appear in the literature, but are inadequate for our construction. For instance, Proposition~7.36 of \citet{BertsekasShreve96} implies that if $A$ is Borel measurable, then $A^\e$ is universally measurable (See Appendix~\ref{app:meas_metric} for more details). However, as discussed earlier in this section, this result does not suffice because we need to show that for a $\sigma$-algebra $\Sigma$, $A\in \Sigma$ implies that $A^\e\in \Sigma$ as well. However, as we detail in Appendix~\ref{app:meas_metric}, this approach shows that for an arbitrary metric space, one can still define the adversarial risk $R^\e$.

%% file: Sections/6-alternative_perturbations.tex
\section{Alternative Models of Perturbations}\label{sec:alt_perturb}
 % There are many other models of perturbations other than additive perturbations in $\Rset^d$. 
 In this paper, we developed techniques for proving the existence of the adversarial Bayes classifier on $\Rset^d$ with additive perturbations. Our techniques could be applied to other natural models of attacks. 
 %in other scenarios. 
 In Appendix~\ref{app:generalize_result}, we state a general theorem that summarizes the part of our theory that is applicable beyond additive perturbations. We discuss three notable examples.
 \begin{example}[Elementwise Scaling]
 \label{ex:small_multipication}
For $\bx\in \Rset^d$, we perturb each coordinate by multiplying it by a number in $[1-\e,1+\e]$. Thus, to perturb $\bx$, we multiply it elementwise by another vector in $B_\e^\infty(\mathbf 1)$.
\end{example}
 
 \citep{engstrom2019exploring} studied the following perturbation empirically in image classification tasks.
 
 \begin{example}[Rotations]
 \label{ex:small_rotations}
Let $\bx\in \Rset^d$. We perturb $\bx$ by multiplying it by a ``small" rotation matrix $\bR$.
We define our perturbation set this time as the set of matrices with 
\[
B = \set[\Big]{\bR\colon \sup_{\|\bx\|_2=1} \bx\cdot \bR\bx\geq 1-\e} .
\]
\end{example}
Our final example is inspired from applications in natural language processing \citep{ebrahimi2018}.
%The following perturbation model is inspired from natural language processing.
% \begin{example}[Discrete Perturbations]
% \label{ex:NLP}
% We consider the attack model proposed in \cite{ebrahimi2018}. We let $\cA$ be our alphabet and we let $X$ be the set of strings of length $\leq N$. We consider perturbations that replace a character at a given index with another character.
% \end{example}
\begin{example}[Discrete Perturbations]
\label{ex:NLP}
Let $\cA$ be an alphabet. For an input string $x$, consider perturbations that replace a character of $x$ at a given index with another character in $\cA$.
\end{example}
The above perturbation models have a lot in common with additive perturbations in $\Rset^d$. All three are examples of \emph{semigroup actions}, and in fact the first two are group actions. Furthermore, all three involve metric spaces. Lastly, denoting a perturbed set as $A^\e$, we still have the containments in \eqref{eq:union_prop_2}.

%for all sequences of sets $A_i$.

Many aspects of the theory developed in this work are applicable in more general scenarios. %and we further discuss in Appendix~\ref{app:generalize_result}. 
In Appendix~\ref{app:motivating_example}, we prove the existence of the adversarial Bayes classifier for a simpler version of Example~\ref{ex:NLP} using the techniques we developed in this paper. Proving the existence of the adversarial Bayes classifier for the other two examples remains an open problem.

 Note that the proof of Theorem~\ref{th:existence_basic} only depends on Lemmas~\ref{lemma:set_limit_2},~\ref{lemma:limsup_liminf_e_commute}, and~\ref{lemma:to_pcr_set}, and not on the properties of $\Rset^d$. Thus in order to generalize our main theorem, one needs to generalize the three lemmas. 
Lemmas~\ref{lemma:limsup_liminf_e_commute} and~\ref{lemma:to_pcr_set} follow directly from the containments in \eqref{eq:union_prop_2}.

Thus it remains to generalize both the measurability considerations and Lemma~\ref{lemma:set_limit_2} on a case-by-case basis. Regarding measurability, we prove a statement similar to Theorem~\ref{th:A^e_univ_meas} in Appendix~\ref{app:measurability} (Theorem~\ref{th:direct_sum_univ_meas_metric}) which applies to perturbations given by a metric ball in a metric space. Specifically, this theorem states that if $A$ is Borel in a metric space, then $A^\e$ is universally measurable. Lastly, our tools may be useful for proving Lemma~\ref{lemma:set_limit_2} in other scenarios.  

%% file: Sections/7-conclusion.tex
\section{Conclusion}
\label{sec:conclusions}
 We initiated the study of fundamental questions regarding the existence of adversarial Bayes optimal classifiers. We provided sufficient conditions that ensure the existence of such classifiers when perturbing by an $\e$-ball. More importantly, our work highlights the need for new tools to understand Bayes optimality under adversarial perturbations, as one cannot simply rely on constructing pointwise optimal classifiers. Our paper also introduces several theorems which could be useful tools in further theoretical work.
 
 Similar to the case of standard loss functions, the most interesting extension of our work is to formulate and study questions related to the consistency of surrogate loss functions for adversarial robustness. We hope that this line of study will lead to new practically useful surrogate losses for designing adversarially robust classifiers.

%% file: Appendices/1-measurability.tex
\section{The Measurability of $A^\e$}\label{app:measurability}
 In this section, we prove two versions of Theorem~\ref{th:A^e_univ_meas}. The first applies to general metric spaces and the second to abstract vector spaces. We discuss the theorem in high generality for two reasons. First, discussing this result in terms of abstract concepts actually clarifies main idea underlying these results. In fact, the proof of the statement we show for metric spaces is simpler than the one we show for vector spaces. Second, we suspect that our framework will be useful in discussing other models of perturbations.
 
 Throughout this section, we denote elements of the vector space $\Rset^d$ in bold ($\bx$) and elements of a general metric space $X$ as non-bold ($x$).

\subsection{Measurability for Metric Spaces}\label{app:meas_metric}
For a measure space $(\nu, X,\cB(X))$ equipped with the Borel $\sigma$-algebra $\cB(X)$, we will denote its completion as $(\overline \nu, X, \cL_\nu(X))$.
 Furthermore, throughout this section, we assume that $X$ is a metric space and that the Borel $\sigma$-algebra $\cB(X)$ is generated by the sets open in the metric on $X$. 

Recall that in Section~\ref{sec:main}, we defined $A^\e$ as $A^\e=A\oplus \overline{B_\e(\zero)}$. Another way to write this relation is 
    \begin{equation}\label{eq:A^eps_def_union}
    A^\e=\bigcup_{\ba\in A}\overline{B_\e(\ba)}.
    \end{equation}
    This form for $A^\e$ is helpful because it allows us to define $A^\e$ for general metric spaces. Notably, one can define the adversairal risk in a general metric space as
    \begin{align*}
        &R^\e(A) = \int(1 - \eta(x)) \one_{A^\e}(x) + \eta(\bx) \one_{(A^C)^\e}(x) \, d\PP\\
        &=\int(1 - \eta(x)) \sup_{a\in A}\one_{\ov{B_\e(a)}}(x) + \eta(x) \sup_{a\in A^C}\one_{\ov{B_\e(a)}}(\bx) \, d\PP
    \end{align*}
    holds for general metric spaces when we define $A^\e$ as in \eqref{eq:A^eps_def_union}. On $\Rset^d$, the second line is equivalent to the expression \eqref{eq:bayes_adv_objective_sup}. We will use this observation later to prove a generalized version of our theorem for alternative models of perturbations.

We start by defining the universal $\sigma$-algebra for a measure space $X$.
\begin{definition}
Let $X$ be a Borel space and let $\sM(X)$ be the set of all finite positive Borel measures on $X$. We define the \emph{universal $\sigma$-algebra} to be 
\begin{equation}\label{eq:universal_sigma_algebra_def}
    \sU(X)=\bigcap_{\nu\in\sM(X)} \cL_\nu(X).    
\end{equation}
If $A\in \sU(X)$, then we say that $A$ is \emph{universally measurable}. \footnote{Alternatively, one could compute the intersection in \eqref{eq:universal_sigma_algebra_def} over all $\sigma$-finite measures. These two approaches are equivalent because for every $\sigma$-finite measure $\lambda$ and compact set $K$, the restriction $\lambda \mres K$ is a finite measure with $\cL_{\lambda\mres K} (X)\supset \cL_\lambda(X)$. See Theorem~1.5 of and Proposition~2.5 \citep{nishiura2010}.}
\end{definition}

%We further assume that the metric space $(X,d)$ is %\emph{proper}. 
%\begin{definition}
%    A metric space $(X,d)$ is \emph{proper} if closed balls are compact.
%\end{definition}
%This assumption is essential as the proof of measurability %relies on the regularity of the space $X$. Specifically, if a space is not proper, it may not be locally compact.

In this section we prove the following theorem:
\begin{theorem}\label{th:direct_sum_univ_meas_metric}\footnote{The original version of this paper stated that in an arbitrary metric space, if $A$ is universally measurable, then $A^\e$ is universally measurable as well, which was an error.}
Let $(X,d)$ be complete separable metric space. Define $A^\e$ as in \eqref{eq:A^eps_def_union}.
If $A\subset X$ is Borel measurable, then $A^\e$ is universally measurable.
\end{theorem}
That the metric on our space $X$ generates the topology on $X$ which in turn generates $\cB(X)$ is implicit in this theorem statement.

This subtlety is crucial when applying Theorem~\ref{th:direct_sum_univ_meas_metric}. For norms $\Rset^d$ however, the situation simplifies-- all norms generate the standard topology. In contrast, a general seminorm does not generate the standard topology on $\Rset^d$, so Theorem~\ref{th:direct_sum_univ_meas_metric} in this case would not apply to $\Rset^d$ with the usual Borel $\sigma$-algebra.

We now describe the basic idea behind the proof of Theorem~\ref{th:direct_sum_univ_meas_metric}. Consider a Borel set $A\subset X$. Then $X\times A$ is Borel in $X\times X$. The set $\Delta_\e=\{(x,y)\in X\times X\colon d(x,y)\leq \e\}$ is closed, and therefore Borel. Thus $X\times A \cap \Delta_\e$ is Borel in $X\times X$. Notice that $A^{\e}$ is the projection of this set onto the first coordinate.  
Such a projection is universally measurable.

In \citep{BertsekasShreve96}, Propositions~7.41 and the statement that $\sU(X)$ contains the analytic $\sigma$-algebra implies the following theorem:
\begin{theorem}\label{th:projection_univ_meas}
    Let $S$ be a Borel set in $X\times Y$ and let $\Pi_1\colon X\times Y\to X$ be projection onto the first coordinate: $\Pi_1(x,y)=x$. Then $\Pi_1(S)$ is universally measurable
\end{theorem}

In fact, this analysis implies that $A^\e$ is measurable with respect to a smaller $\sigma$-algebra called the \emph{analytic $\sigma$-algebra}. See Chapter~7 of \citep{BertsekasShreve96} for details. We formally perform this calculation below.

\begin{proof}[Proof of Theorem~\ref{th:direct_sum_univ_meas_metric}]
Let $\Delta_\e=\{(x,y)\colon d(x,y)\leq \e\}$.
We will show that $A^\e=\Pi_1(X\times A \cap \Delta_\e)$. Theorem~\ref{th:projection_univ_meas} will then imply the result. 
\begin{align*}
    \Pi_1(X\times A\cap \Delta_\e)&=\{\bx\colon \text{for some }\ba\in A, (\bx,\ba)\in \Delta_\e\}=\{\bx\colon \text{for some }\ba\in A, d(\ba,\bx)\leq \e\}\\
    &=\bigcup_{\ba\in A}\ov{B_\e(\ba)}=A^\e
\end{align*}

\end{proof}

\subsection{Measurability for Vector Spaces}\label{app:meas_vector}

%Again in the definition of the universal $\sigma$-algebra, we follow the treatment of \citep{nishiura2010}.
%We start by defining the universal $\sigma$-algebra for a measure space $X$.
%\begin{definition}
%Let $X$ be a Borel space and let $\sM(X)$ be the set of all $\sigma$-finite Borel measures on $X$. We define the \emph{universal $\sigma$-algebra} to be 
%\begin{equation}\label{eq:universal_sigma_algebra_def}
%    \sU(X)=\bigcap_{\substack{nu\in\sM(X)\\\nu(X)<\infty}} \cL_\nu(X).    
%\end{equation}

%If $A\in \sU(X)$, then we say that $A$ is \emph{universally measurable}.\footnote{Alternatively, one could compute the intersection in \eqref{eq:universal_sigma_algebra_def} over all $\sigma$-finite measures. These two approaches are equivalent because for every $\sigma$-finite measure $\lambda$ and compact set $K$, the restriction $\lambda \mres K$ is a finite measure with $\cL_{\lambda\mres K} (X)\supset \cL_\lambda(X)$. See Theorem~1.5 of and Proposition~2.5 \citep{nishiura2010} imply that when $X$ is an \emph{absolute measurable space}, for an argument that these two notions agree.}
%\end{definition}

In this section we show the following measurability result:
\begin{theorem}\label{th:direct_sum_univ_meas}
Let $(X,\|\cdot\|)$ be a separable vector space. Define $A^\e$ as $A^\e=A\oplus \ov{B_\e(\zero)}$, where $B_\e$ is an $\e$-ball in the norm $\|\cdot\|$.
If $A\subset X$ is universally measurable, then $A^\e$ is universally measurable as well.
\end{theorem}
As $\Rset^d$ with the standard topology is separable and all norms on $\Rset^d$ generate the standard topology, Theorem~\ref{th:A^e_univ_meas} immediately follows from Theorem~\ref{th:direct_sum_univ_meas}.

Again, that the norm on our space $X$ generates the topology on $X$ which in turn generates $\cB(X)$ is implicit in this theorem statement.

Before proving Theorem~\ref{th:direct_sum_univ_meas}, we define another useful concept.
\begin{definition}
Let $X,Y$ be a separable metric spaces and let $(\overline \nu, Y,\cL_\nu(Y))$ be a complete $\sigma$-finite measure space. Then $X$ is \emph{absolute measurable} if for every injective continuous map $h\colon X\to Y$, $h(X)$ is an element of $\cL_\nu(Y)$. 
\end{definition}

This definition is useful due to the following theorem:
\begin{theorem}\label{th:Darst-Grz}
Let $X$ be an absolute measurable Borel space and $Y$ a separable metrizable space. Let $f\colon X\to Y$ be a homeomorphism. 

Then 
\[f[\sU(X)]\subset \sU(Y).\]
\end{theorem}
This theorem is the implication (1) $\Rightarrow$ (4) of the Purves-Darst-Grzegorek Theorem, stated on page 33 in Chapter~2.1 of \citep{nishiura2010}. A separable vector space is $\sigma$-compact. This fact implies that that Theorem~\ref{th:Darst-Grz} applies.

\begin{lemma}\label{lemma:sigma_compact_abs}
A $\sigma$-compact space is absolute measurable.
\end{lemma}

We now describe the basic idea behind the proof of Theorem~\ref{th:direct_sum_univ_meas} for $\Rset^d$. Consider the homeomorphism $w\colon \ov{B_\e(\zero)}\times \Rset^d\to \ov{B_\e(\zero)}\times \Rset^d $ given by $w(\bv,\bx)=(\bv,\bx+\bv)$. Then for any set $A$, $w(\ov{B_\e(\zero)}, A)= \ov{B_\e(\zero)}\times A^\e$. Therefore, if $\ov{B_\e(\zero)}$ is universally measurable in $\ov{B_\e(\zero)}\times \Rset^d$, then $\ov{B_\e(\zero)} \times A^\e$ is also universally measurable. To conclude that $A^\e$ is universally measurable in $\Rset^d$, it remains to show that that $\ov{B_\e(\zero)}\times S$ is universally measurable in $\ov{B_\e(\zero)}\times \Rset^d $ iff $S$ is universally measurable in $\Rset^d$.

\begin{lemma}\label{lemma:prod_univ_meas}
Let $X,Y$ be Borel spaces. If $S\in \sU(Y)$, then $X\times S\in \sU(X\times Y)$.
\end{lemma}

\begin{lemma}\label{lemma:prod_univ_meas_converse}
	Let $X,Y$ be second countable, locally compact Hausdorff spaces. Assume that $X$ is compact and $Y$ is $\sigma$-compact. If $X\times S\in \sU(X\times Y)$, then $S\in \sU(Y)$.
\end{lemma}

We prove Theorem~\ref{th:direct_sum_univ_meas} using these results.
\begin{proof}[Proof of Theorem~\ref{th:direct_sum_univ_meas}]
 
	Consider the function $w\colon \ov{B_\e(\zero)}\times X \to \ov{B_\e(\zero)}\times X$ given by $w(\bh,\bx)=(\bh,\bx+\bh)$. Then $w$ is continuous and invertible, so it is a homeomorphism. Furthermore, it maps the set $ \ov{B_\e(\zero)}\times A$ to $ \ov{B_\e(\zero)}\times (A\oplus \ov{B_\e(\zero)})$. 
	
	Let $A$ be a universally measurable subset of $X$. Then by Lemma~\ref{lemma:prod_univ_meas}, $\ov{B_\e(\zero)}\times A$ is universally measurable in $\ov{B_\e(\zero)}\times X$. Thus, Theorem~\ref{th:Darst-Grz} implies that $w(\ov{B_\e(\zero)},A)=\ov{B_\e(\zero)}\times A^\e $ is universally measurable in $\ov{B_\e(\zero)}\times X$. Lastly, Lemma~\ref{lemma:prod_univ_meas_converse} implies that $A^\e$ is measurable.

\end{proof}
\subsection{Proofs of Lemmas~\ref{lemma:sigma_compact_abs},~\ref{lemma:prod_univ_meas},~and~\ref{lemma:prod_univ_meas_converse}}
\begin{replemma}{lemma:sigma_compact_abs}
A $\sigma$-compact space is absolute measurable.
\end{replemma}
\begin{proof}[Proof of Lemma~\ref{lemma:sigma_compact_abs}]
We will start by showing that a compact space is absolute measurable. Let $H$ be a compact topological space and let $Y$ be a separable metric space. If $f\colon H\to Y$ is continuous, a well-known theorem from topology implies that $f(H)$ is compact as well. A compact subset of a metric space is always closed, and therefore $f(H)$ is a Borel set.

Next, consider a $\sigma$-compact space $X$. Write 
\[X=\bigcup_{n\in \Nset} H_n\] where each $H_n$ is compact. Then if $f\colon X\to Y$ is a continuous map, then $f(X)$ is a countable union of Borel sets; 
\[f(X)=\bigcup_{n\in \Nset} f(H_n)\]
and is therefore Borel as well.
\end{proof}

\begin{replemma}{lemma:prod_univ_meas}
Let $X,Y$ be Borel spaces. If $S\in \sU(Y)$, then $X\times S\in \sU(X\times Y)$.
\end{replemma}
\begin{proof}[Proof of Lemma~\ref{lemma:prod_univ_meas}]
Let $(\nu,X\times Y, \cB(X\times Y))$ be an arbitrary finite Borel measure on $X\times Y$ and let $S\in \sU(Y)$. We will show that $X\times S\in \cL_\nu(X\times Y)$. As the universal $\sigma$-algebra is the intersection of all $\cL_\nu(X\times Y)$ for all finite Borel measures $\nu$, this inclusion will imply that $X\times S\in \sU(X\times Y)$.

Let $\lambda$ be the marginal distribution on $Y$ given by $\lambda(B)=\nu(X\times B)$ with $\sigma$-algebra $\cB(Y)$. Now consider the completion $(\overline \lambda, Y, \cL_\lambda(Y))$. Because $S$ is in the universal $\sigma$-algebra for $Y$, we know that $S\in \cL_\lambda(Y)$. Therefore, $S=B\cup N'$ where $B$ is a Borel set and $N'$ is a subset of a null Borel set $N$. Because $N$ is Borel, $X\times N$ is as well and $\nu(X\times N)=\lambda(N)=0$. Therefore, $X\times N$ is a null Borel set for the measure space $(\nu,X\times Y, \cB(X\times Y))$. Thus both $X\times N'$ and $X\times B$ are in the complete measure space $(\overline \nu, X\times Y, \cL_\nu(X\times Y))$. Therefore, $X\times S=X\times B\cup X\times N'$ is in $\cL_\nu(X\times Y)$ as well. 
\end{proof}

However, to prove the converse to Lemma~\ref{lemma:prod_univ_meas}, one must apply the concept of regularity of measures.

 \begin{definition}\label{def:inner_regular}
Let $\tau$ be a topology on a set $X$ and $\cB(X)$ the Borel $\sigma$-algebra generated by $\tau$. Let $\PP$ be a Borel measure on $(X,\cB(X))$.Then $\PP$ is \emph{inner regular} if for all measurable sets $E$,
$\PP(E) = \sup \set{\PP(K)\colon K\subset E, K\text{ compact}}$.
    %Similarly, $\PP$ is \emph{outer regular} if for all measurable sets $E$, $\PP(E) = \inf \set{\PP(U)\colon U\supset E, K\text{ open}}$.
A space is \emph{regular} if all finite Borel measures on $X$ are inner regular.
\end{definition}

Theorem~7.8 of \citep{folland} implies that most measure spaces encountered in applications are regular:

 \begin{theorem}\label{th:Radon}
Let $X$ be a second-countable and locally compact Haudsorff space. Then every finite Borel measure is inner regular. 
\end{theorem}

The notion of regularity extends to complete measures. 
\begin{lemma}\label{lemma:regularity_complete}
    Let $\nu$ be a finite positive Borel measure on a regular space $X$ and let $\ov{\nu}$ be the completion of $\nu$. Let $A\in \cL_\nu(X)$. Then 
    \[\ov \nu(A)=\sup_{\substack{K\subset A\\ K\text{ compact}}} \nu(K).\]
\end{lemma}
\begin{proof}
    If $A\in \cL_\nu(X)$, then there is a Borel set $B$ with $B\subset A$ and $\ov\nu(A)=\nu(B)$. The result then follows from the definition of inner regularity for Borel measures.
\end{proof}

\begin{replemma}{lemma:prod_univ_meas_converse}
	Let $X,Y$ be second countable, locally compact Hausdorff spaces. Assume that $X$ is compact and $Y$ is $\sigma$-compact. If $X\times S\in \sU(X\times Y)$, then $S\in \sU(Y)$.
\end{replemma}
\begin{proof}[Proof of Lemma~\ref{lemma:prod_univ_meas_converse}]
	As $X,Y$ are both $\sigma$-compact spaces, Theorem~\ref{th:Radon} implies that $X\times Y$ is regular. Fix a Borel probability measure $\lambda$ on $X$, and let $\nu$ be any finite Borel measure on $Y$. Then $\lambda\times \nu$ is a Borel probability measure on $X\times Y$, so it is inner regular. Let $\ov{\lambda\times \nu}$ be the completion of $\lambda\times \nu$. Then
	\[\ov{\lambda\times \nu} (X\times S)= \sup_{\substack{ K \text{ compact}\\ K\subset X\times S}} \lambda\times \nu(K) \]
	We will now argue that 
	\begin{equation}\label{eq:compact_to_product}
		\sup_{\substack{ K \text{ compact}\\ K\subset X\times S}} \lambda\times \nu(K)=\sup_{\substack{ K \text{ compact}\\ K\subset  S}} \nu(K)
	\end{equation}
	Let $K\subset X\times S$ and let $\Pi_2\colon X\times X\to X$ be projection onto the second coordinate. Because the continuous image of a compact set is compact, $K'=\Pi_2(K)$ is compact and contained in $S$. Thus $X\times S\supset X\times K'\supset X\times K$, which implies \eqref{eq:compact_to_product}. Now \eqref{eq:compact_to_product} applied to $S^C$ implies that 
	\[\ov{\lambda\times \nu}(X\times S)=\inf_{\substack{ U^C \text{ compact}\\ U\supset X\times S}} \lambda\times \nu(U) =\inf_{\substack{ U^C \text{ compact}\\ U\supset S}}  \nu(U)\]
	
	Thus 
	\[\sup_{\substack{ K \text{ compact}\\ K\subset  S}} \nu(K)=\inf_{\substack{ U^C \text{ compact}\\ U\supset S}}  \nu(U):=m\]
	Let $K_n$ be a sequence of compact sets contained in $A$ for which $\lim_{n\to \infty} \nu(K_n)=m$ and $U_n$ a sequence of sets containing $A$ for which $U_n^C$ is compact and $\lim_{n\to \infty} \nu(U_n)=m$. Because a finite union of compact sets is compact, one can choose such sequences that satisfy $K_{n+1}\supset K_n$ and $U_{n+1}\subset U_n$. Then $D=\bigcup K_n$, $V=\bigcap U_n$ are Borel sets that satisfy $D\subset S\subset V$ and $\nu(D)=\nu(V)$, so $V-D $ a null set. Thus $S-D$ is a subset of the null Borel set $V-D$, so $S\in \cL(\Rset^d)$. As $\nu$ was arbitrary, it follows that $S$ is universally measurable.

\end{proof}

%% file: Appendices/2-set_limit_exists.tex
\section{Proof of Lemma~\ref{lemma:set_limit_2} }\label{app:set_limit_proof}
In the next subsection we show how Lemma~\ref{lemma:regularity_ball} implies Lemma~\ref{lemma:set_limit_2}. The proof of Lemma~\ref{lemma:regularity_ball} is delayed to~\ref{app:set_limit_lemmas}. 
\subsection{Main Argument}
The following lemma assists in understanding the convergence of pseudo-certifiably robust sets.

	\begin{lemma}\label{lemma:ball_open_union_set_lemma}
		Let $\{\ba_n\}$ be a sequence converging to $\ba$. Then if $\bx\in B_\e(\ba)$, then for sufficiently large $n$, $\bx \in B_\e(\ba_n)$.
	\end{lemma}
	\begin{proof}
		Set $r=\|\bx-\ba\|$. Then if we choose $n$ large enough so that $\|\ba-\ba_{n}\|< \e -r$, then 
		\[\|\bx-\ba_{n}\|\leq \|\bx-\ba\|+\|\ba-\ba_{n}\|< r+(\e-r)=\e\]
		Therefore, for sufficiently large $n$, $\bx \in B_\e(\ba_n)$.
	\end{proof}

    The following result is central to the proof of Lemma~\ref{lemma:set_limit_2}.
    \begin{replemma}{lemma:regularity_ball}
            Let $\mu$ be Lebesgue measure and let $S\subset \Rset^d$. 
    If for each $\bs\in \partial S$ there exists a ball $B_\e(\ba)$ with $B_\e(\ba)\subset S$ and $\bs\in \partial B_\e(\ba)$, then $\mu(\partial S)=0$.
\end{replemma}
    The next lemma shows that if $\tls A_n=\tli A_n$, then for every $\bx\in \liminf A_n^\e$ there is a ball for which $\bx \in \ov{B_\e(\ba)}$ but $B_\e(\ba)\subset \liminf A_n^\e$. Thus the set $\liminf A_n^\e$ satisfies the property required by Lemma~\ref{lemma:regularity_ball} at all points. It remains to show that this property also holds on the boundary $\partial \liminf A_n^\e$, which is later accomplished by taking limits.

    \begin{lemma}\label{lemma:liminf_A^e_well_behaved}
        Let $A_n$ be a sequence for which $\tli A_n=\tls A_n$. Then $\liminf A_n^\e$ has the following property: for every $\bx \in \liminf A_n^\e$, there is a ball $B_\e(\ba)$ for which $\bx \in \ov{B_\e(\ba)}$ and $B_\e(\ba)\subset \liminf A_n^\e$.
    \end{lemma}
    \begin{proof}
        Let $\bx \in \liminf A_n^\e$. The expression for the $\liminf$ in \eqref{eq:liminf_sequence_def} implies that there is a $J$ for which $\bx\in A_{n}^\e$ for all $n>N$. Hence one can write $\bx=\ba_{n}+\bh_{n}$ with $\ba_n\in A_n$ and $\bh_n\in \ov{B_\e(\zero)}$ for all $n>N$. Now pick a subsequence $n_{j}$ for which $\bh_{n_{j}}$ converges and set $\bh=\lim_{j\to\infty} \bh_{n_{j}}$. Then let 
    \begin{equation}\label{eq:ak_def}
        \ba=\lim_{j\to \infty} \ba_{n_{j}}= \bx-\bh.
    \end{equation}
    Due to the definition of $\tls$ in \eqref{eq:tls_def}, $\ba\in \tls A_{n}$ and by the assumption on our sequence $A_n$, $\tls A_n=\tli A_{n}$. Thus there is a sequence $\td \ba_{n}$ for which $\td \ba_{n}\in A_{n}$ 
    and $\lim_{n\to \infty} \td \ba_{n}=\ba$. Then $B_\e(\td \ba_{n})\subset A_{n}^\e$ and 
    Lemma~\ref{lemma:ball_open_union_set_lemma} then implies that $B_\e(\ba)\subset \liminf_j A_{n}^\e$. 
    
    Lastly, one can conclude that $\|\bx-\ba\|\leq \e$ from the definition of $\ba$ in \eqref{eq:ak_def}.
    \end{proof}

Finally, Lemma~\ref{lemma:set_limit_2} is a consequence of Lemma~\ref{lemma:regularity_ball}, Lemma~\ref{lemma:liminf_A^e_well_behaved}, and Theorem~\ref{th:set_limit_Rockafellar}.

\begin{replemma}{lemma:set_limit_2}
        Let $\QQ$ be a finite positive measure and assume that $\QQ$ is absolutely continuous with respect to Lebesgue measure. For any sequence of sets $A_n$, there is a sub-sequence $A_{n_j}$ for which 
    \[\limsup A_{n_j}^\e\dequal \liminf A_{n_j}^\e\]
\end{replemma}
\begin{proof}
    Let $\mu$ denote Lebesgue measure. We will find a subsequence $A_{n_j}$ of $A_n$ for which $\mu(\limsup A_{n_j}-\liminf A_{n_j})=0$. By Theorem~\ref{th:set_limit_Rockafellar}, one can find a subsequence $n_j$ for which $\tli A_{n_j}=\tls A_{n_j}$. Then for this subsequence, 
    Lemma~\ref{lemma:liminf_A^e_well_behaved} then applies to this subsequence. We will argue that $\liminf A_n^\e$ in fact satisfies the property of Lemma~\ref{lemma:regularity_ball}. 
    
     Let $\bx \in \partial \liminf A_{n_j}^\e$. We will find a ball $B_\e(\ba)\subset \liminf A_{n_j}^\e$ for which $\bx\in \ov{B_\e(\ba)}$. If $\bx \in \liminf A_{n_j}^\e$, then there is a sequence $\bx^k\in \liminf A_{n_j}^\e$ converging to $\bx$. 
    By Lemma~\ref{lemma:liminf_A^e_well_behaved}, for each $\bx^k$, there is a $\ba^k$ with $\bx\in \ov{B_\e(\ba^k)}$ and $B_\e(\ba^k)\subset \liminf A_{n_j}^\e$.
    Furthermore, because $\bx^k\to \bx $, the set $\{\ba^k\}$ is bounded. Let $k_m$ be a subsequence for which $\ba^{k_m}$ converges and set $\ba=\lim_{m\to\infty} \ba^{k_m}$. Then because $B_\e(\ba^{k_m})\subset \liminf A_{n_j}^\e$, Lemma~\ref{lemma:ball_open_union_set_lemma} implies that $B_\e(\ba)\subset \liminf A_{n_j}^\e$ as well. Next,
    \[\|\bx-\ba\|\leq \|\bx-\bx^{k_m}\|+\|\bx^{k_m}-\ba^{k_m}\|+\|\ba^{k_m}-\ba\|\leq \|\bx-\bx^{k_m}\|+\e+\|\ba^{k_m}-\ba\|\]
    As $\|\bx-\bx^{k_m}\|$ , $\|\ba-\ba^{k_m}\|$ both approach zero, it follows that $\bx\in \ov{B_\e(\ba)}$. Therefore, Lemma~\ref{lemma:regularity_ball} applies.
    
\end{proof}

\subsection{Proof of Lemma~\ref{lemma:regularity_ball}}\label{app:set_limit_lemmas}

To prove Lemma~\ref{lemma:regularity_ball} we take an approach that is standard in geometric measure theory. The strategy is to apply the Lebesgue differentiation theorem. 
\begin{theorem}[Lebesgue Differentiation Theorem]
Assume that $f\colon\Rset^d\to \Rset$ is bounded. Then the following holds for $\bx$ $\mu$-a.e.:

\[\lim_{r\to 0}\frac 1 {\mu(B_r^2(\bx))} \int_{B_r^2(\bx)} f d\mu=f(\bx)\]
\end{theorem}

See \citet{folland} for a proof. We prove a slightly more general version of Lemma~\ref{lemma:regularity_ball}:

\begin{lemma}\label{lemma:regularity}
Let $\mu$ be Lebesgue measure and let $S\subset \Rset^d$. 
If for each $\bs\in \partial S$ there exists an open convex $C$ with $C\subset S$ and $\bs\in \partial C$, then $\mu(\partial S)=0$.
\end{lemma}

\begin{proof}[Proof of Lemma~\ref{lemma:regularity}]
We will apply the Lebesgue differentiation theorem to the function $\one_{\partial S}$. 

Let $B_r^2(\bx)$ denote the radius $r$ ball centered at $\bx$ given by the 2-norm. The Lebesgue differentiation theorem implies that the set defined by
\[E=\left\{ \bx\colon \lim_{r\to 0}\frac 1 {\mu(B_r^2(\bx))} \int_{B_r^2(\bx)} \one_{\partial S}(\by) dy\neq \one_{\partial S}(\bx)\right\}\]
has measure zero.
We will show $\partial S\subset E$, which will imply $\mu(\partial S)=0$.

This amounts to showing that for $\bx\in \partial S$,

\[\lim_{r\to 0}\frac 1 {\mu(B_r^2(\bx))} \int_{B_r^2(\bx)} \one_{\partial S}(\by) dy=\lim_{r\to 0} \frac{\mu(\partial S\cap B_r^2(\bx))}{\mu(B_r^2(\bx))}\neq 1\]

Specifically, we will show that for sufficiently small $r$, there exists a constant $K>0$ independent of $r$ for which 
\[\frac{\mu(\interior S\cap B_r^2(\bx))}{\mu(B_r^2(\bx))}\geq K>0.\]
This inequality will imply the result as 
\[\lim_{r\to 0} \frac{\mu(\partial S\cap B_r^2(\bx))}{\mu(B_r^2(\bx))}\leq 1-\liminf_{r\to 0} \frac{\mu(\interior S\cap B_r^2(\bx))}{\mu(B_r^2(\bx))}-\liminf_{r\to 0} \frac{\mu(\interior S^C\cap B_r^2(\bx))}{\mu(B_r^2(\bx))}.\]

Pick $\bx_0\in \partial S$. Then by assumption, there is an open convex set $C$ for which $\bx_0\in \partial C$ and $C\subset S$. As $C$ is open, $C\subset \interior S$. Furthermore, $B_1^2(\bx_0)\cap C$ is non-empty, open, and convex. Thus we can pick $d$ points $\bx_1\ldots \bx_d\in C$ for which the vectors $\{\bx_1-\bx_0\ldots \bx_d-\bx_0\}$ are linearly independent. By the convexity of $C\subset \interior S$, the interior of the convex hull of $\{\bx_0\ldots \bx_d\}$ is contained in $\interior S$. We will call the interior of this convex hull $T$. By construction, for any $r$, $ T\cap B_r^2(\bx_0)$ is disjoint from $\partial S$ and contained in $S$.
This implies 
\[\frac { \mu(\interior S\cap B_r^2(\bx))}{\mu( B_r^2(\bx))}\geq \frac { \mu( T\cap B_r^2(\bx))}{\mu( B_r^2(\bx))}\]

%\[\frac 1 {\mu(B_r^2(\bx))} \int_{B_r^2(\bx)} \one_{D}(y) dy=\frac { \mu(D\cap B_r^2(\bx))}{\mu( B_r^2(\bx))}= 1-\frac { \mu(D^C\cap B_r^2(\bx))}{\mu( B_r^2(\bx))}\geq 1-\frac { \mu(T\cap B_r^2(\bx))}{\mu( B_r^2(\bx))}\]

We we will show that for $r<\min_{i\in [1,n]}\|\bx_i-\bx_0\|$, 
\[\frac { \mu(T\cap B_r^2(\bx))}{\mu( B_r^2(\bx))}\geq K>0\] for some constant $K$.

Specifically, if $r<\min_{i\in [1,n]}\|\bx_i-\bx_0\|$, $B_r^2(\bx)$ contains $\bx_0+r\frac{\bx_i-\bx_0}{\|\bx_i-\bx_0\|}$ for each $i$. Then because $B_r^2(\bx)$ is convex, it must contain the simplex defined by these vectors which we will call $W$. See Figure~\ref{fig:simplex} for an illustration. 
\begin{figure}\label{fig:simplex}
\centering
\includegraphics[scale=0.6]{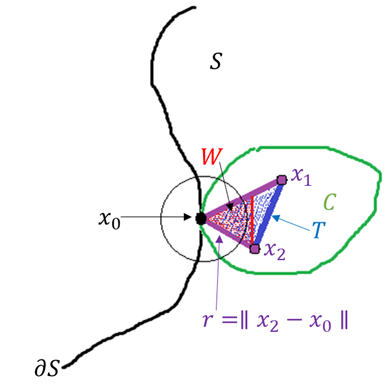}
\caption{The convex set $C$, the simplex $T$, and the simplex $W$ in two dimensions. The illustrated ball has radius less than $r$. In this Figure, $r=\|\bx_2-\bx_0\|$.}
\end{figure}
A standard calculation shows that $\mu(W)=\frac {r^d} {n!} |\det(M)|$, where 
\[M=\left[\frac{\bx_1-\bx_0}{\|\bx_1-\bx_0\|}\ldots \frac{\bx_n-\bx_0}{\|\bx_n-\bx_0\|} \right]\]
see for instance \citet{stein1966}.

Therefore, as the volume of a ball radius $r$ is $\frac {\pi^{\frac d2}r^d}{\Gamma(\frac d2+1)}$, we have shown that for $r<\min_{i\in [1,n]}\|\bx_i-\bx_0\|$, 
\[\frac { \mu(T\cap B_r^2(\bx))}{\mu( B_r^2(\bx))}\geq \frac{\mu(W)}{\mu(B_r^2(\bx))}= \frac{\frac {r^d} {d!} |\det M|}{\frac {\pi^{\frac d2}}{\Gamma(\frac d2+1)}r^d}= \frac{\Gamma(\frac d 2 +1) |\det M|}{d! \pi^{\frac d 2}}>0\]
\end{proof}

%% file: Appendices/3-e-e_properties.tex
\section{Properties of the $\empty^\e,\empty^{-\e}$ Operations}\label{app:properties}\label{app:properties_basic}
    In this section, we will discuss some basic properties of the $\empty^\e,\empty^{-\e}$ operations. We will apply these results throughout the rest of the appendix. Furthermore, this section should highlight some of the intuition for working with these set operations. 
    
    We will adopt \eqref{eq:A^eps_def_union} as our definition for $A^\e$.

    This convention will allow us to generalize much of the results in this section to arbitrary metric spaces.
    After defining $A^\e$ we can then define $A^{-\e}$ as 
    \begin{equation}\label{eq:A^-eps_def_union}
        A^{-\e}=((A^C)^\e)^C.
    \end{equation}
        %Throughout this appendix we will work in $\Rset^d$, however, most of our proofs work for an arbitrary metric space where $A^\e,A^{-\e}$ are defined as \eqref{eq:A^eps_def_union} and \eqref{eq:A^-eps_def_union}.

    The following lemma details how the $\empty^\e$ and $\empty^{-\e}$ operations interact with unions and intersections.
    
    \begin{lemma}\label{lemma:eps_set_relations_R^d}
    Define $A^\e$ as in \eqref{eq:A^eps_def_union} and $A^{-\e}$ as \eqref{eq:A^-eps_def_union}. 
    Then for any sequence of sets $\{A_i\}$, the following set containments hold:

    \noindent\begin{minipage}{.5\linewidth}
    \begin{equation}\label{eq:eps_cup}
        \bigcup_{i=1}^\infty A_i^\e=\left[\bigcup_{i=1}^\infty A_i\right]^\e
    \end{equation}
\end{minipage}
\begin{minipage}{.5\linewidth}
     \begin{equation}\label{eq:-e_out}
        \bigcap_{n\geq 1} A_n^{-\e} = \paren[\Big]{\bigcap_{n\geq 1} A_n}^{-\e}
    \end{equation}
\end{minipage}
\begin{minipage}{.5\linewidth}
    \begin{equation}\label{eq:eps_cap}
        \bigcap_{i=1}^\infty A_i^\e \supset\left[\bigcap_{i=1}^\infty A_i\right]^\e
    \end{equation}
\end{minipage}
\begin{minipage}{.5\linewidth}
    \begin{equation}\label{eq:-eps_cup}
    \bigcup_{i=1}^\infty A_i^{-\e} \subset         \left[\bigcup_{i=1}^\infty A_i\right]^{-\e}
    \end{equation}
\end{minipage}\par\vspace{\belowdisplayskip}
\end{lemma}
\begin{proof}
\textbf{Showing \eqref{eq:eps_cup}:}\\
For any set $A$, one can write
\begin{equation*}A^\e=\bigcup_{\ba\in A} \overline{B_\e( \ba)}.\end{equation*}
Thus 
\[\bigcup_{i=1}^\infty A_i^\e=\bigcup_{i=1}^\infty \bigcup_{\ba\in A_i} \overline{B_\e(\ba) }=\bigcup_{\ba\in \bigcup_{i=1}^\infty A_i}\overline{B_\e(\ba)}=\left(\bigcup_{i=1}^\infty A_i\right)^\e.\]

\textbf{Showing \eqref{eq:eps_cap}:}\\
First note that if $C\supset B$, then $C^\e\supset B^\e$.
Next, since $A_i\supset \bigcap_{i=1}^{\infty} A_i$, 
\[A_i^\e\supset \left(  \bigcap_{j=1}^{\infty} A_j\right)^\e\] for all $i$. Thus \eqref{eq:eps_cap} holds.

\textbf{Showing \eqref{eq:-e_out}:}\\
Recall that $A^{-\e}=((A^C)^\e)^C$. If we apply \eqref{eq:eps_cup} to $(A_i^C)^\e$, we get that
\[\bigcup_{i=1}^\infty \left(A_i^C\right)^\e=\left(\bigcup_{i=1}^\infty A_i^C\right)^\e= \left(\left(\bigcap_{i=1}^\infty A_i\right)^C\right)^\e.\] Now upon taking complements, 
\[\left(\bigcap_{i=1}^\infty A_i\right)^{-\e}=\left(\bigcup_{i=1}^\infty \left(A_i^C\right)^\e\right)^C=\bigcap_{i=1}^\infty \left(\left(A_i^C\right)^\e\right)^C=\bigcap_{i=1}^\infty A_i^{-\e}.\]

\textbf{Showing \eqref{eq:-eps_cup}:}
If we apply \eqref{eq:eps_cap} to $A_i^C$, then 
\[\bigcap_{i=1}^\infty \left(A_i^C\right)^\e\supset \left(  \bigcap_{i=1}^{\infty} A_i^C\right)^\e= \left( \left( \bigcup_{i=1}^{\infty} A_i\right)^C\right)^\e.\] Taking complements gives \eqref{eq:-eps_cup}.
\end{proof}

Next, we use the previous representations to show that $F(A^\e)=\emptyset$ and $F((A^{-\e})^C)=\emptyset$, where we define $F(\cdot)$ in \eqref{eq:define_F(A)}.
\begin{lemma}\label{lemma:F_A^eps_empty}
    For a set $A$, define
        \begin{equation*}
            F(A)=\{\bx\in A:\text{ every closed $\e$-ball containing $\bx$ also intersects }A^C\}\tag{\ref{eq:define_F(A)}}
        \end{equation*}
    Then
    \begin{equation}\label{eq:F_A^eps_empty}                    F(A^\e)=\emptyset 
    \end{equation}
    \begin{equation}\label{eq:F_A^-eps_C_empty}                 F((A^{-\e})^C)=\emptyset 
    \end{equation}
\end{lemma}
This lemma is an important stepping stone towards showing that there exists a pseudo-certifiably robust adversarial Bayes classifier.
\begin{proof}[Proof of Lemma~\ref{lemma:F_A^eps_empty}]
Equation~\ref{eq:A^eps_def_union}
implies that each point $\bx$ in $A^\e$ is included in some closed $\e$-ball that is contained in $A^\e$. Subsequently, the definition of $F$ in \eqref{eq:define_F(A)} implies \eqref{eq:F_A^eps_empty}. Lastly, \eqref{eq:F_A^-eps_C_empty} follows by applying \eqref{eq:F_A^eps_empty} to $(A^{-\e})^C$.
\end{proof}

The next lemma provides an alternative interpretation of the $\empty^\e,\empty^{-\e}$ operations.
\begin{lemma}\label{lemma:more_about_-eps}
Define $A^\e,A^{-\e}$ as in \eqref{eq:A^eps_def_union},\eqref{eq:A^-eps_def_union}. Then alternative characterizations of $A^\e,A^{-\e}$ are given by
\begin{equation}\label{eq:A^eps_description}
A^\e=\{\bx\in X: \overline{B_\e(\bx)}\cap A\neq \emptyset\}
\end{equation}

\begin{equation}\label{eq:A^-eps_description} A^{-\e}=\{\ba: \overline{B_\e(\ba)}\subset A\}\end{equation}

\end{lemma}
Notice that in $\Rset^d$ \eqref{eq:A^-eps_description} reduces to
\[A^{-\e}=\{\ba\in A: \ba+\bh\in A \text{ for all }\bh \text{ with } \|\mathbf \bh\|\leq \e\}\]
\begin{proof}[Proof of Lemma~\ref{lemma:more_about_-eps}]
\textbf{Showing \eqref{eq:A^eps_description}:}\\
Recall that $\bz\in A^\e$ iff for some $\ba\in A$, $\bz\in \overline{B_\e(\ba)}$. However, 
\[\bz\in \overline{B_\e(\ba)} \Leftrightarrow \ba\in \overline{B_\e(\bz)}\Leftrightarrow \overline{B_\e(\bz)}\text{ intersects }A\]

\textbf{Showing \eqref{eq:A^-eps_description}:}\\
Recall the definition $A^{-\e}= ((A^C)^\e)^C$. Then 
\begin{align*} 
 & \ba\in A^{-\e}\\
&\Leftrightarrow \ba\not \in (A^C)^\e\\
&\Leftrightarrow \overline{B_\e(\ba)} \text{ does not intersect }A^C\text{ (by \eqref{eq:A^eps_description})}\\
&\Leftrightarrow  \overline{B_\e(\ba)}\subset A 
\end{align*}

\end{proof}

%% file: Appendices/4-regularity_proof.tex
\section{Proof of Lemma~\ref{lemma:to_pcr_set} }
\label{app:sequence_original}\label{app:properties_succession}
 
In some of our proofs, we apply the $\empty^\e$ and $\empty^{-\e}$ operations to sets multiple times in succession. In this section, we describe how the $\empty^\e$ and the $\empty^{-\e}$ operations interact. These considerations turn out to be important because applying $\empty^\e$ followed by $\empty^{-\e}$ to a set (or vice versa) decreases the adversarial loss. We prove this statement in Lemmas~\ref{lemma:to_pcr_set} and~\ref{lemma:decreaseR_e(A)}, which are the central conclusions of this Appendix.

Our first result states that applying $\empty^{-\e}$ an then $\empty^{\e}$ to a set $A$ makes the set smaller while applying $\empty^{\e}$ and then $\empty^{-\e}$ makes the set larger.

 \begin{lemma}\label{lemma:A_e_-e_containment}
Define the $\empty^\e,\empty^{-\e}$ operations as in \eqref{eq:A^eps_def_union}, \eqref{eq:A^-eps_def_union}. Then 
\begin{equation}\label{eq:supset}
(A^\e)^{-\e}\supset A
\end{equation}
\begin{equation}\label{eq:subset}
(A^{-\e})^\e\subset A
\end{equation}
\end{lemma}
\begin{proof}
 To start, note that \eqref{eq:subset} follows from applying \eqref{eq:supset} to $A^C$ and then taking complements. 
 
 In order to show \eqref{eq:supset}, we make use of Equation~\ref{eq:A^-eps_description}.
 Equation~\ref{eq:A^-eps_description} implies that if $\bx\in A^{-\e}$, then $\overline{B_\e(\bx)}\subset A$. As 
 \[(A^{-\e})^\e=\bigcup_{\bx\in A^{-\e}} \overline{B_\e(\bx)}\] and each $\overline{B_\e(\bx)}$ is entirely contained in $A$, the entire set $(A^{-\e})^\e$ is contained in $A$ as well.
\end{proof}

 \begin{lemma}
 \label{lemma:characterize_eps}
Define $A^\e,A^{-\e}$ as in \eqref{eq:A^eps_def_union},\eqref{eq:A^-eps_def_union}. Then the following hold:

\begin{align}
&A = (A^{-\e})^\e\sqcup F(A)\label{eq:A_-eps_eps}\\
&(A^\e)^{-\e} = A\sqcup F(A^C).\label{eq:A_eps_-eps}
 \end{align}
\end{lemma}
 Specifically, \eqref{eq:A_-eps_eps} implies that $(A^{-\e})^\e=A- F(A)$ and \eqref{eq:A_eps_-eps} implies that $(A^\e)^{-\e} =A\cup F(A^C)$.
 Figure~\ref{fig:F(A)_F(A^C)} illustrates the sets $F(A)$ and $F(A^C)$.

 \begin{proof}[Proof of Lemma~\ref{lemma:characterize_eps}]
\mbox{}\\\textbf{Showing $\supset$ for \eqref{eq:A_-eps_eps}:}\\
It's clear that $F(A)\subset A$ and Lemma~\ref{lemma:A_e_-e_containment} implies that $(A^{-\e})^\e\subset A$ as well. \\
\textbf{Showing $\subset$ for \eqref{eq:A_-eps_eps}:}\\
We will prove that $A-F(A)\subset (A^{-\e})^\e$. Assume that $\bx\in A-F(A)$. Then there is a closed $\e$-ball containing $\bx$ that does not intersect $A^C$, which means that this ball is completely contained in $A$. Thus for some $\ba\in A,$ $\bx\in \overline{B_\e(\ba)}\subset A$. Thus by \eqref{eq:A^-eps_description}, $\ba\in A^{-\e}$. Furthermore, $\bx\in \overline{B_\e(\ba)}$ implies that $\bx\in (A^{-\e})^\e$.

\textbf{Showing disjoint union for \eqref{eq:A_-eps_eps}:}

Lemma~\ref{lemma:A_e_-e_containment} states that $(A^{-\e})^\e\subset A$. Specifically, every point in $(A^{-\e})^\e$ is contained in a closed $\e$-ball that is contained in $A$. As no point in $F(A)$ satisfies this property, $(A^{-\e})^\e$ and $F(A)$ are disjoint. 

\textbf{Showing \eqref{eq:A_eps_-eps}:}

Applying \eqref{eq:A_-eps_eps} to $A^C$ results in
\[A^C=((A^C)^{-\e})^\e\sqcup F(A^C)=((A^\e)^C)^\e\sqcup F(A^C).\] Taking complements of both sides of this equation produces
\[A= (A^\e)^{-\e}\cap F(A^C)^C\]
and therefore
\[A\cup \left( (A^\e)^{-\e} \cap F(A^C)\right)=(A^\e)^{-\e}.\]
The union is actually a disjoint union because $F(A^C)\subset A^C$ which is disjoint from $A$. It remains to show that $F(A^C)\subset (A^\e)^{-\e}$, so that $F(A^C)\cap (A^\e)^{-\e}=F(A^C)$. 

We now show that $F(A^C)\subset (A^\e)^{-\e}$. Pick $\bx \in F(A^C)$. We will show that for every $\by\in \overline{B_\e(\bx)}$, $\by \in A^\e$. This statement will imply that $\overline{B_\e(\bx)}\subset A^\e$ and then \eqref{eq:A^-eps_description} will then imply that $\bx\in (A^\e)^{-\e}$.

If $\by\in \overline{B_\e(\bx)}$, then $\overline{B_\e(\by)}$ contains $\bx$. By definition, because $\bx\in F(A^C)$, every ball containing $\bx$ intersects $A$. Therefore $\ov{B_\e(\by)}$ intersects $A$ and then \eqref{eq:A^eps_description} then implies that $\by\in A^\e$. 

\end{proof}
In the previous lemma, we characterized $(A^{-\e})^\e$ and $(A^{\e})^{-\e}$, in terms of $A$ and $F(\cdot)$ but this characterization is a little complicated. Here, we show that if in fact $A=B^{-\e}$ some set $B$, then $(A^{\e})^{-\e}$ simplifies. Similarly, $(A^{-\e})^\e$ simplifies if in fact $A=B^\e$ for some set $B$.

\begin{lemma}
\label{lemma:eps_-eps_repeats}
For any set $A$, the following hold:
\begin{equation*}
\Big((A^\e)^{-\e}\Big)^\e = A^\e, \quad\quad
\Big((A^{-\e})^\e\Big)^{-\e} = A^{-\e}.
\end{equation*}
\end{lemma}

\begin{proof}[Proof of Lemma~\ref{lemma:eps_-eps_repeats}]
By Lemmas~\ref{lemma:F_A^eps_empty} and~\ref{lemma:characterize_eps}, 
\[\Big((A^\e)^{-\e}\Big)^\e=(\Big(A^\e\Big)^{-\e})^\e=A^\e-F(A^\e)=A^\e.\]
Similarly,
\[\Big((A^{-\e})^\e\Big)^{-\e}=(\Big(A^{-\e}\Big)^\e)^{-\e}=A^{-\e}\cup F((A^{-\e})^C)=A^{-\e}. \]
\end{proof}

We next prove a short lemma that will help us understand how the $-\e,\e$ operations reduce the adversarial loss.
\begin{lemma}\label{lemma:BCD_contain}
Let $\empty^\e,\empty^{-\e}$ be as in \eqref{eq:A^eps_def_union} and \eqref{eq:A^-eps_def_union}. Consider a set $B\subset X$. Then if $D=(B^{-\e})^\e$ and $C=(B^\e)^{-\e}$, then $C^\e\subset B^\e, C^{-\e}\supset B^{-\e}$ and $D^\e\subset B^\e,D^{-\e}\supset B^{-\e}$.
\end{lemma}
\begin{proof}
First consider the set $D$. Then by Lemma~\ref{lemma:eps_-eps_repeats}, $D^{-\e}=B^{-\e}$. Furthermore, according to Lemma~\ref{lemma:A_e_-e_containment}, $D\subset B$, so that $D^\e\subset B^\e$. 
 
 Next, according to Lemma~\ref{lemma:eps_-eps_repeats}, $C^\e= B^\e$. Furthermore, according to Lemma~\ref{lemma:A_e_-e_containment}, $C\supset B$, so that $C^{-\e}\supset B^{-\e}$.
\end{proof}
Lastly, we prove a lemma which states that applying the $\empty^{\e},\empty^{-\e}$ operations in succession decreases the adversarial loss. Observe that $R^\e$ incurs a penalty of 1 on both $F(A)$ and $F(A^C)$ because points in these sets are always within $\e$ of a point with the opposite class label.

 \begin{lemma}
 \label{lemma:decreaseR_e(A)}
For any set $A$, the following hold:
\begin{align}
&R^\e(A)\geq R((A^\e)^{-\e}) \label{eq:A_e_-e_less}\\
&R^\e(A)\geq R((A^{-\e})^\e) \label{eq:A_-e_e_less}.\end{align}
\end{lemma}

\begin{proof}[Proof of Lemma~\ref{lemma:decreaseR_e(A)}]
The basic idea here is that the maximum penalty is incurred on $F(A)$, so removing $F(A)$ from $A$ and adding it to $A^C$ will not increase the loss. (Compare the statement of this lemma with Lemma~\ref{lemma:characterize_eps} and Figure~\ref{fig:F(A)_F(A^C)}.)  The same holds for $F(A^C)$ and $A^C$.

Let $B=(A^{-\e})^\e$ or $B=(A^\e)^{-\e}$. Lemma~\ref{lemma:BCD_contain} implies that $B^\e\subset A^\e$ and $B^{-\e}\supset A^{-\e}$. These containments imply the result because if $B^\e\subset A^\e$ and $B^{-\e}\supset A^{-\e}$ then 
\[\eta(\bx) \one_{A^\e}+(1-\eta(\bx)) \one_{(A^C)^\e}\geq \eta(\bx) \one_{B^\e}+(1-\eta(\bx)) \one_{(B^C)^\e}\]
holds pointwise, so 
\[R^\e(A)= \int \eta(\bx) \one_{A^\e}+(1-\eta(\bx)) \one_{(A^C)^\e}d\PP\geq\int \eta(\bx) \one_{B^\e}+(1-\eta(\bx)) \one_{(B^C)^\e} d\PP =R^\e(B).\]
\end{proof}

By taking $B=(A^{-\e})^\e$, $E=(A^\e)^{-\e}$, Lemma~\ref{lemma:to_pcr_set} immediately follows from Lemma~\ref{lemma:decreaseR_e(A)} and the definition of the $\empty^\e$ operation.
\begin{replemma}{lemma:to_pcr_set}
            Let $A$ be any set. Then there exist sets $B,E$ for which $B$ and $E^C$ are pseudo-certifiably robust and $R^\e(B)\leq R^\e(A)$, $R^\e(E)\leq R^\e(A)$.
\end{replemma}

%% file: Appendices/5-liminf_limsup_swap.tex
\section{Proof of Lemma~\ref{lemma:limsup_liminf_e_commute} and a Generalization (Lemma~\ref{lemma:limsup_liminf_e_commute_gen})}\label{app:to_decreasing_sequence_proof}
 We begin by reviewing some results of Appendix~\ref{app:properties}. To start, recall that the operation $A^\e=A\oplus \overline{B_\e(\zero)}$, satisfies the relations of~\eqref{eq:union_prop_2}:
 \begin{equation*}
    \left(\bigcup_{i=1}^\infty A_i\right)^\e=\bigcup_{i=1}^\infty A_i^\e \quad , \quad
\left(\bigcap_{i=1}^\infty A_i\right)^\e\subset\bigcap_{i=1}^\infty A_i^\e     \tag{\ref{eq:union_prop_2}}
 \end{equation*}
   In the next section, we will prove a version of Lemma~\ref{lemma:limsup_liminf_e_commute} for other models of perturbations. 
Thus, in the rest of this appendix, rather than focusing on $\Rset^d$, we will assume that $\empty^\e$ is a set operation that satisfies \eqref{eq:union_prop_2}. This formulation will allow us to prove an existence theorem for other models of perturbations. As elements of our space $X$ are not necessarily vectors, we write them in non-bold font ($x$). We now state a generalized version of Lemma~\ref{lemma:limsup_liminf_e_commute}. 
\begin{lemma}\label{lemma:limsup_liminf_e_commute_gen}
Let $A^\e $ be any set operation that satisfies \eqref{eq:union_prop_2}. Then 
     \[\limsup A_n^\e \supset \left( \limsup A_n\right)^\e \quad \text{and}\quad \liminf A_n^\e \supset \left( \liminf A_n\right)^\e\]
 \end{lemma}

 Note that Lemma~\ref{lemma:limsup_liminf_e_commute} is simply Lemma~\ref{lemma:limsup_liminf_e_commute_gen} combined with the fact that $A^\e$ defined as $A\oplus \overline{B_\e(\zero)}$ satisfies \eqref{eq:union_prop_2} (shown in Lemma~\ref{lemma:eps_set_relations_R^d}). 
 \begin{proof}[Proof of Lemma~\ref{lemma:limsup_liminf_e_commute_gen}]
 We start by proving the statement for $\limsup$. By \eqref{eq:union_prop_2},
 \[\limsup A_n^\e=\bigcap_{N=1}^\infty \bigcup_{n=N}^\infty A_n^\e=\bigcap_{N=1}^\infty \left(\bigcup_{n=N}^\infty A_n\right)^\e\supset \left(\bigcap_{N=1}^\infty \bigcup_{n=N}^\infty A_n\right)^\e \]
 The statement for $\liminf$ follows from a similar argument: 
 \[\liminf A_n^\e=\bigcup_{N=1}^\infty \bigcap_{n=N}^\infty A_n^\e\supset \bigcup_{N=1}^\infty \left(\bigcap_{n=N}^\infty A_n\right)^\e= \left(\bigcup_{N=1}^\infty \bigcap_{n=N}^\infty A_n\right)^\e \]
 \end{proof}

%% file: Appendices/6-generalized_perturbations.tex
\section{More General Results}\label{app:generalize_result}
In this Appendix, we present a generalization of our main result. This generalization concerns other models of perturbations. As discussed in Section~\ref{sec:alt_perturb}, there are many other possible models of perturbations in adversarial learning. A more general result would help address the existence of the adversarial Bayes classifier in these scenarios as well. We provide a motivating example in the next subsection. 

\begin{theorem}\label{th:main_general}
    Let $X$ be a separable metric space and let $\cB(X),\sU(X)$ be the corresponding Borel and universal $\sigma$-algebras respectively. Let $\PP$ be the completion of a measure on $\mathcal B(X)$ restricted to $\sU(X)$. For $A\subset X$, let $\empty^\e\colon A\to A^\e$ be a set operation for which $A^\e$ is universally measurable for all sets $A\in \sU(X)$. Furthermore, assume that $\empty^\e$ satisfies the properties 
\begin{align}
 &\bigcup_{n\in \Nset} A_n^\e=\left(\bigcup_{n\in \Nset} A_n\right)^\e\label{eq:eps_cup_gen}\\
 &\bigcap_{n\in \Nset} A_n^\e \supset \left(\bigcap_{n\in \Nset} A_n\right)^\e\label{eq:eps_cap_gen}
\end{align} 
for every sequence of sets $\{A_n\}$.
Define the loss
\[R^\e(A)=\int(1-\eta(x))\one_{A^\e}(x)+\eta(x)\one_{(A^C)^\e} d\PP\]
Assume that for some minimizing sequence $A_n$, one can always find  a subsequence $A_{n_j}$ for which $\limsup A_{n_j}^\e\dequal \liminf A_{n_j}^\e$, where $\dequal$ is with respect to the measure $\PP$.
 Then there exists a minimizer to $R^\e$ in the $\sigma$-algebra $\sU(X)$.
 \end{theorem}
If $A^\e$ is defined by perturbations in a metric space, Theorem~\ref{th:direct_sum_univ_meas_metric} could be used to conclude that $A^\e$ is universally measurable.
 
 Just like Theorem~\ref{th:existence_basic}, one could also define the $\empty^{-\e}$ operation as $A^{-\e}=((A^C)^\e)^C$, and then argue that there exists a pseudo-certifiably robust minimizer, if $A^\e$ is universally measurable for universally measurable $A$. This statement follows from the same argument as Lemma~\ref{lemma:decreaseR_e(A)}.

 \subsection{A Motivating Example--Applying Theorem~\ref{th:main_general}}
 \label{app:motivating_example}
 To show the utility of Theorem~\ref{th:main_general}, we present an application inspired by NLP. For clarity, we choose a model of discrete perturbations somewhat simpler than Example~\ref{ex:NLP}. Let $X$ be all strings of finite length with a finite alphabet $\cA$. This space is countable and therefore separable. Furthermore, this space is discrete. Recall that in a discrete space, every set is measurable.
 Hence, the Borel $\sigma$-algebra consists of all subsets of $X$, which implies that $\sU(X)$ and $\cB(X)$ are equal. 
 
 We will define our perturbations as swapping two letters in a string at specified positions. Formally, for $w\in X$, let $|w|$ denote the length of the string. Furthermore, let $T$ be the set of functions defined by 
 \begin{align*}
     &T=\left \{b^{i,j}\colon X\to X \Big| b^{i,j}(w)_k=w_k\text{ if }k\neq i,j \text{ or }\max(i,j)>|w|,\right.\\
     &\left.b^{i,j}(w)_i=w_j, b^{i,j}(w)_j=w_i\text{ otherwise}\right\}.
 \end{align*}

 In other words, $b^{i,j}$ will swap the letters at $i$ and $j$ in $w$ if $w$ has length at least $\max(i,j)$ and will keep the string fixed otherwise. Now let $B$ be a finite subset of $T$.
 
 If $A$ is a set of strings, we define 
 \[A^\e=\{b(a)\colon a\in A,b\in B\}.\]

 To start, note that for this definition of the $\empty^\e$ operation, we still have that for every sequence of sets $A_n$, the relations \eqref{eq:eps_cup_gen}, \eqref{eq:eps_cap_gen} hold. The proofs are the same as \eqref{eq:eps_cup}, \eqref{eq:eps_cap} of
 Lemma~\ref{lemma:eps_set_relations_R^d}, so we do not reproduce it here.
 
 Next, we argue that for every sequence $A_n$, there is a subsequence $A_{n_j}$ for which $\liminf A_{n_j}=\limsup A_{n_j}$. The proof is similar to that of Theorem~\ref{th:set_limit_Rockafellar} presented in \citep{RockafellarWets1998}.

 Let $A=\limsup A_n$. Because the space $X$ is countable, one can enumerate $A=\limsup A_n= \{a_n\}_{n=1}^N$, with $N\in \Nset \cup \{\infty\}$. Now we inductively define $N$ nested subsequences of sets $\{A_n^k\}_{n=1}^\infty$ indexed by $k$ as follows:
Because $a_1\in \limsup A_n$, by the characterization of $\limsup$ in \eqref{eq:limsup_sequence_def}, one can find a subsequence $A_{n_m}$ for which $a_1\in A_{n_m}$ for all $m$. Let $A^1_m=A_{n_m}$.

Now given the sequence $\{A_n^k\}$, we inductively define $\{A_n^{k+1}\}$. Consider the element $a_{k+1}$ of the sequence $A=\{a_n\}_{n=1}^N$. If $a_{k+1} \in \limsup_n A_n^k$, then one can find a subsequence $A_{n_m}^k$ for which $a_{k+1}\in A_{n_m}^k$ for all $m$. Similarly, if $a_{k+1} \not \in  \limsup_n A_n^k$, then one can find a subsequence $A_{n_m}^k$ for which $a_{k+1}\not \in A_{n_m}^k$ for all $m$. Thus we choose the subsequence $A_{n_m}^k$ so that either $a_{k+1}\in A_{n_m}^k$ for all $m$ or $a_{k+1}\not \in A_{n_m}^k$ for all $m$. Set $A_{n}^{k+1}=A^k_{n_m}$.

If $N=|\limsup A_n|$ is finite, consider the sequence $A_n^N$. Then $\limsup A_n^N\subset \limsup A_n$. Furthermore, by the construction of the sequences $\{A_n^k\}_{n=1}^N$, for each $a\in \limsup A_n$, either $a \in A_n^N$ for all $n$ or $a\not \in A_n^N$ for all $n$. This observation implies that if $a\in \limsup_n A_n^N$, then $a$ is eventually in the tail of the sequence $A_n^N$ and thus $a\in \liminf_n A_n^N$. Therefore $\liminf_n A_n^N=\limsup_N A_n^N$.

We now analyze the case $N=\infty$. Consider now the diagonal sequence $A_{n_k}=A_k^k$. Again, we have the containment $\limsup_k A_k^K\subset \limsup_n A_n$. Let  $a_j$ be an element of $ A=\limsup A_n$. Then by construction of the sequences $\{A_n^k\}_{n=1}^\infty$, either $a_j\in A_k^k$ for all $k\geq j$ or $a_j\not \in A_k^k$ for all $k\geq j$. Thus every $a\in \limsup A_{n_k}$ is also in $A_{n_j}$ for all sufficiently large $j$ so $a\in \liminf_k A_{n_k}$. Therefore, $\limsup_{k} A_{n_k}=\liminf_k A_{n_k}$.
 
 \subsection{Proving Theorem~\ref{th:main_general}}
 The proof of Theorem~\ref{th:main_general} follows the same steps as the proof of Theorem~\ref{th:existence_basic}. As mentioned in Section~\ref{sec:alt_perturb}, the big picture motivation is that Theorem~\ref{th:existence_basic} followed directly from Lemmas~\ref{lemma:set_limit_2},~\ref{lemma:limsup_liminf_e_commute}, and~\ref{lemma:to_pcr_set} -- we did not use properties of $\empty^\e$ or the space $\Rset^d$ outside of these three Lemmas. The main challenge is generalizing these concepts. With the proper definitions, the proof of Theorem~\ref{th:main_general} is exactly the same as the proof of Theorem~\ref{th:existence_basic}, except that we replace Lemma~\ref{lemma:set_limit_2} with the assumption that for every minimizing sequence $A_n$ there exists a subsequence $A_{n_j}$ for which $ \limsup A_{n_j}^\e- \liminf A_{n_j}^\e$ is a null set.

\begin{proof}[Proof of Theorem~\ref{th:main_general}]
    Let $A_n$ be a minimizing sequence of $R^\e$. Pick a subsequence $A_{n_j}$ for which 
    
    \begin{equation}\label{eq:subsequence_convergence_gen}
        \PP(\limsup A_{n_j}^\e-\liminf A_{n_j}^\e)=0
    \end{equation}
    and set $A=\limsup A_{n_j}$. Then 
    \begin{align*}
        \inf_{A\text{ Borel}} R^\e(A)=&\lim_{j\to \infty} R^\e(A_j)\geq \int \liminf_{j\to \infty} \left(\eta \one_{A_{n_j}^\e} +(1-\eta) \one_{(A_{n_j}^C)^\e}\right) d\PP &\text{(Fatou's Lemma)}\\
        &\geq \int  \eta \liminf_{j\to \infty}\one_{A_{n_j}^\e} +(1-\eta) \liminf_{j\to \infty} \one_{(A_{n_j}^C)^\e} d\PP \\
        &=\int \eta \one_{\limsup_j A_{n_j}^\e} +(1-\eta) \one_{\liminf_j(A_{n_j}^C)^\e} d\PP &\text{(Equation~\ref{eq:subsequence_convergence_gen})}\\
        &\geq \int \eta \one_{\left(\limsup_j A_{n_j}\right)^\e} +(1-\eta) \one_{\left(\liminf_j A_{n_j}^C\right)^\e} d\PP &\text{(Lemma~\ref{lemma:limsup_liminf_e_commute_gen})}\\
        &=\int \eta \one_{\left(\limsup_j A_{n_j}\right)^\e} +(1-\eta) \one_{\left((\limsup_jA_{n_j})^C\right)^\e} d\PP \\
        &=R^\e(A)
    \end{align*}
    Therefore, $A$ is a minimizer of $R^\e$.
\end{proof}